\documentclass[11pt]{article}

\usepackage[margin=1in]{geometry}
\usepackage{amsmath,amssymb,amsfonts,amsthm}
\usepackage{bm}
\usepackage{graphicx}
\usepackage{booktabs}
\usepackage{url}
\usepackage{enumerate}
\usepackage[round,authoryear]{natbib}  
\usepackage{amsmath,amssymb,amsfonts}
\usepackage{bm}
\usepackage{mathtools}
\usepackage{booktabs}
\usepackage{url}
\usepackage{enumerate}
\usepackage{multirow}
\usepackage{microtype}

\DeclareMathOperator*{\argmin}{arg\,min}

\theoremstyle{plain}
\newtheorem{theorem}{Theorem}[section]
\newtheorem{lemma}[theorem]{Lemma}
\newtheorem{proposition}[theorem]{Proposition}
\newtheorem{corollary}[theorem]{Corollary}

\theoremstyle{definition}

\newtheorem{remark}[theorem]{Remark}
\newtheorem{assumption}[theorem]{Assumption}

\newcommand{\PP}{\mathbb{P}}
\newcommand{\RR}{\mathbb{R}}
\newcommand{\XX}{\mathbf{X}}
\newcommand{\xx}{\mathbf{x}}

\newcommand{\bbeta}{\boldsymbol{\beta}}
\newcommand{\bxi}{\boldsymbol{\xi}}
\newcommand{\E}{\mathbb{E}}
\newcommand{\1}{\mathbf{1}}

\title{Elite-Driven Support Vector Machines for Classification}
\author{%
  Mohammad Jafari Jozani\thanks{Department of Statistics, University of Manitoba, Winnipeg, Manitoba, CANADA,  R3T2N2, Email:  \texttt{m\_jafari\_jozani@umanitoba.ca}}\and
  Bahram Moeinianfar\thanks{Department of Statistics, University of Manitoba, Winnipeg, Manitoba, CANADA, R3T2N2}
}
\date{\today}

\begin{document}

\maketitle

\begin{abstract}
Support vector machines (SVMs) are a standard tool for binary classification,
but their classical formulations are purely data-driven and offer no direct way
to encode trusted benchmark models or structured preferences on selected subsets
of the data. We propose Elite-Driven Support Vector Machines (EDSVM), a general
framework that augments regularized empirical risk minimization by guiding the
slack variables for a curated set of elite observations (typically the union of
support vectors from one or more reference SVMs). EDSVM combines the usual slack
loss with a deviation penalty that shrinks new slacks toward benchmark slack
values, defining a localized, margin-aligned notion of proximity to reference
models, unlike global function penalties in knowledge distillation or
teacher-student methods, and without requiring privileged features as in
SVM+/LUPI. Within this framework we develop two concrete models, C-EDSVM and
LS-EDSVM, based respectively on hinge-type and squared-slack losses. For both
variants we derive dual quadratic programs that can be implemented with modest
modifications of standard SVM solvers, and we give simple sufficient conditions
under which the induced margin losses are classification calibrated. Simulation
studies and experiments on several UCI benchmarks show that EDSVMs closely track
the behaviour induced by reference SVMs while achieving predictive performance
that is competitive with, and sometimes better than, C-SVM, LINEX-SVM, and
LS-SVM.
\end{abstract}


%
%

\section{Introduction}

Support vector machines (SVMs), introduced in their modern form by \citet{cortes1995svm} and rooted in earlier work by \citet{vapnik1964perceptrons}, are a cornerstone of statistical learning and machine learning \citep{vapnik1995nature}. In binary classification, SVMs construct a separating hyperplane with maximal margin between two classes and, through kernel methods, extend naturally to nonlinear decision boundaries in high-dimensional feature spaces. Their combination of geometric interpretability, convex optimization, and strong risk guarantees has led to widespread use in applications such as image recognition, bioinformatics, text mining, and medical diagnosis \citep[see, for example,][]{hastie2009esl,scholkopf2002kernels,moguerza2006svmApplications,steinwart2008svm}.

In the standard formulation, an SVM estimator is determined entirely by the training sample, a margin-based surrogate loss, and a regularization penalty; all observations enter the objective symmetrically. This yields attractive theoretical guarantees, but offers limited flexibility for incorporating structured prior information. In many applications, practitioners possess information that is not naturally encoded in the standard SVM objective. They might have access to benchmark classifiers fitted under alternative losses, kernels, or tuning schemes; expert-identified observations that are especially reliable or decision-critical; or scientific, regulatory, and ethical requirements concerning behaviour near selected subgroups or client segments. Such situations suggest that certain reference decision rules and certain training points should play a privileged role in constructing a new classifier.

Our starting point is that, in soft-margin SVMs, the local geometry of the decision boundary is governed by points on or inside the margin. These observations, which carry nonzero or borderline slack, encode much of the information about the classifier near the boundary. We refer to them as elite observations. Benchmark SVMs fitted under alternative losses, kernels, or tuning parameters produce their own elite sets and associated slacks, summarizing how those reference classifiers resolve ambiguity near the decision boundary. Rather than forcing the entire discriminant function $f$ to remain close to a benchmark $f^*$ across the whole feature space, we use these benchmark slacks only on a curated elite subset.

\noindent\textit{Positioning relative to related work.}
Our work is related to several lines of research that incorporate auxiliary information from reference models or privileged sources. In knowledge distillation and teacher-student methods, a student model is trained to mimic the predictive behaviour of one or more teachers by penalizing discrepancies between their outputs over the input space; see, for example, \citet{hinton2015distill,lopezpaz2016gendist}. Vapnik's learning using privileged information (LUPI) and SVM+ \citep{vapnik2009lupi,vapnik2015lupi} use additional training-only features to parametrize or correct slack variables in an auxiliary space. More broadly, benchmark-guided regularization often penalizes function-level discrepancies such as $\|f-f^*\|^2$, thereby enforcing global agreement even far from the decision boundary.

EDSVM differs from these approaches by anchoring the slack variables of the new classifier to benchmark slacks only on a selected elite set, while leaving the remaining empirical risk unchanged. This produces a localized, margin-aligned notion of proximity to benchmark models, does not require privileged features, and preserves the convex optimization structure of standard SVM training. To further clarify the role of auxiliary information, it is helpful to contrast EDSVM with several nearby formulations. In weighted SVMs, observation-specific penalties or class weights scale the loss coefficients, producing a global reweighting of the empirical objective without directly matching a benchmark model. In distillation or teacher-guided methods, auxiliary information enters through penalties on discrepancies between student and teacher outputs, again enforcing global agreement in function or output space. In SVM+ / LUPI, privileged features available only at training time are used to model or correct the slack variables through an auxiliary feature space. By contrast, EDSVM does not rely on privileged covariates and does not impose global function matching; instead, it anchors the new slack variables to benchmark slack values only on a curated elite subset, thereby introducing a localized form of guidance at the margin level on selected influential observations.

\noindent\textit{Contributions.}
The main contributions of this paper are as follows.
\begin{enumerate}[(i)]
\item We formulate a general elite-driven SVM (EDSVM) loss at the slack level, in which a conventional margin-based slack loss is combined with a deviation penalty that anchors the new model to benchmark slack values on a selected elite set. This yields a statistically interpretable regularization mechanism. The weight $\omega \in (0,1)$ plays the role of a shrinkage parameter that continuously interpolates between pure empirical risk minimization $(\omega=1)$ and pure benchmark imitation $(\omega \to 0)$, in a manner localized to the margin-relevant observations. The resulting framework is a principled extension of regularized empirical risk minimization with an explicit, tunable shrinkage toward trusted reference models.

\item Within this framework we develop two concrete instances: C-EDSVM, based on the hinge loss with an additional slack deviation term, and LS-EDSVM, which augments the least-squares SVM formulation with guided slack shrinkage. For both variants we derive dual representations, show how elite information modifies the underlying quadratic program, and give sufficient conditions under which the induced margin losses are classification calibrated.

\item We establish Rademacher complexity and excess-risk bounds for both C-EDSVM and LS-EDSVM (Section~\ref{sec:theory}) that have a clear statistical interpretation. The bounds decompose into three transparent components: (a) a bias term $\frac{1-\omega}{\omega}\frac{m}{n}\mathcal{E}_m^*$ proportional to the benchmark mismatch $\mathcal{E}_m^*$ on an elite set of size $m$, quantifying the statistical cost of using an imperfect benchmark; (b) a variance term controlled by the RKHS norm of the benchmark comparator; and (c) a complexity term driven by the effective radius of the function class, which can be strictly smaller than in standard SVM when $m\mathcal{E}_m^* < n\hat{R}_\phi(f_\phi^*)$ (Corollary~\ref{cor:vs-svm}). Thus, EDSVM is statistically preferable to standard SVM when benchmark quality is sufficiently high relative to the empirical risk of the baseline comparator. Under Tsybakov noise and sufficiently accurate benchmarks, the approach also preserves minimax-optimal fast rates (Theorem~\ref{thm:minimax}).

\item We analyze how elite information modifies the dual quadratic program through both the kernel matrix and the linear term, while preserving convexity and allowing implementation with only modest modifications of standard SVM solvers. The framework also accommodates multiple benchmark models via aggregation functionals on benchmark slacks.

\item Through simulation studies and experiments on several benchmark datasets, we illustrate that EDSVMs deliver robust and interpretable shrinkage toward benchmark classifiers and often achieve predictive performance that is competitive with, and in some settings better than, C-SVM, LINEX-SVM, and LS-SVM.
\end{enumerate}

Taken together, these contributions position EDSVM as a methodology for benchmark-anchored large-margin learning. In many applications, the objective is not only to maximize predictive efficiency, but also to construct a classifier whose behaviour remains close to trusted reference information on a curated set of influential observations. EDSVM addresses this need by encoding such information at the slack level, thereby preserving the large-margin structure while introducing localized guidance near the decision boundary. Our developed theory clarifies when this guidance is statistically useful. If the benchmark slacks are accurate on the elite set, so that $\mathcal{E}_m^*$ is small, and if the elite set is not too large relative to the sample size, then the benchmark-bias term is controlled and the effective complexity radius can be reduced. When $\omega=1$, the guidance term is removed and the corresponding standard SVM-type estimator is recovered. When $\mathcal{E}_m^*=0$, the benchmark-bias term in the generalization bound vanishes, although this reflects benchmark fidelity rather than equality of estimators for $\omega<1$.

The remainder of the paper is organized as follows. Section~\ref{sec:svm} briefly reviews standard SVM methodology and its formulation as regularized empirical risk minimization with margin-based losses. Section~\ref{sec:edsvm} introduces the EDSVM framework, formalizes the notion of elite observations, and develops the slack-guided loss together with the C-EDSVM and LS-EDSVM models. Section~\ref{sec:theory} develops generalization and excess-risk results for both C-EDSVM and LS-EDSVM, with more extensive analysis for the C-EDSVM case. Section~\ref{sec:numerical} presents simulation studies designed to probe the behaviour of EDSVMs under controlled conditions, and Section~\ref{sec:realdata} reports empirical results on several real datasets. We conclude in Section~\ref{sec:discussion} with a summary of our findings and a discussion of possible extensions and future research directions.

\section{Background on Support Vector Machines}
\label{sec:svm}

Let $\mathbf{X} \in \mathbb{R}^p$ denote a feature vector and $Y \in \{-1,+1\}$ a binary class label.  Consider a linear classifier based on a hyperplane in $\mathbb{R}^p$,
\[
  C_f(\mathbf{x}) = \operatorname{sign}\big(f(\mathbf{x})\big),  \qquad  f(\mathbf{x}) = \beta_0 + \mathbf{x}^\top \boldsymbol{\beta},
\]
with intercept $\beta_0 \in \mathbb{R}$ and normal vector $\boldsymbol{\beta} \in \mathbb{R}^p$.  The decision regions are
\[
  H_{+1} = \{\mathbf{x} : C_f(\mathbf{x}) = +1\},
  \qquad
  H_{-1} = \{\mathbf{x} : C_f(\mathbf{x}) = -1\},
\]
and the decision boundary is
\[
  \mathcal{D}_f = \{\mathbf{x} : f(\mathbf{x}) = 0\}.
\]
Given training data $\mathcal{T} = \{(\xx_i, y_i)\}_{i=1}^n \subset \RR^p \times \{-1,+1\}$, many hyperplanes may perfectly separate the classes whenever linear separability holds.\ The basic idea of support vector machines is to select, among all separating hyperplanes, the one that maximizes the margin, i.e., the minimal distance of any training point to the decision boundary \citep{cortes1995svm,vapnik1998slt}. In the ideal separable case, the hard-margin SVM solves
\begin{align*}
\min_{\beta_0,\bbeta} ~ \frac{1}{2} \|\bbeta\|_2^2 \quad 
\text{subject to} \quad  y_i \big( \beta_0 + \xx_i^\top \bbeta \big) \ge 1,
\quad i = 1,\ldots,n,
\end{align*}
so that maximizing the margin is equivalent to minimizing $\|\bbeta\|_2^2$ under correct classification constraints. In most applications, the data are not strictly linearly separable, and violations of the margin are allowed via slack variables $\xi_i \ge 0$, leading to the soft-margin formulation
\begin{equation}
\label{eq:softmargin}
\min_{\beta_0,\bbeta,\bxi} ~  
\frac{1}{2} \|\bbeta\|_2^2 + C \sum_{i=1}^n \xi_i \quad 
\text{subject to} \quad 
y_i\big(\beta_0 + \xx_i^\top \bbeta\big) \ge 1 - \xi_i, ~~\text{and}~~\xi_i \ge 0, \quad i=1,\ldots, n,
\end{equation}
where $C>0$ controls the trade-off between margin maximization and misclassification penalties.\ Points with $\xi_i>0$ lie strictly inside the margin or are misclassified. More generally, the observations that determine the fitted separator are the support vectors, which in the soft-margin SVM include points on the margin as well as points inside it; equivalently, they are characterized by nonzero dual coefficients. Problem~\eqref{eq:softmargin} admits an equivalent regularized empirical risk form based on the hinge loss $\varphi_{\mathrm{hinge}}(z) = \max(0, 1 - z)$, as follows
\begin{equation}
\label{eq:hingeERM}
\min_{\beta_0,\bbeta} 
\quad \frac{\lambda}{2} \|\bbeta\|_2^2 
+ \sum_{i=1}^n \max(0, 1-y_i(\beta_0 + \xx_i^\top \bbeta)),
\end{equation}
with $\lambda = 1/C$.\  We refer to this method as C-SVM throughout the paper \citep{scholkopf2002kernels} .\  More generally, one can replace the hinge loss by any convex, margin-based surrogate $\varphi_1 : \RR \to \RR_+$ and consider
\begin{equation}
\label{eq:generalERM}
\min_{\beta_0,\bbeta} 
~\sum_{i=1}^n \varphi_1\big(y_i(\beta_0 + \xx_i^\top \bbeta)\big)
+ \frac{\lambda}{2} \|\bbeta\|_2^2,
\end{equation}
where different choices of $\varphi_1$ lead to different decision boundaries and possibly different sets of support vectors; see, for example, \citet{bartlett2006convexity,hastie2009esl,steinwart2008svm}.\ In the SVM and statistical learning literature it is common to require that $\varphi_1$ be classification calibrated \citep{bartlett2006convexity}, meaning that minimizing the associated surrogate risk yields decision rules that are Bayes-optimal for the underlying $0$--$1$ loss.
In other words, optimizing the convex surrogate does not lead to systematically wrong decision boundaries, and in the large-sample limit one can recover the Bayes classifier by minimizing the $\varphi_1$-risk.\ The following sufficient condition for classification calibration of margin-based losses is standard; see, e.g., \citet{bartlett2006convexity}.

\begin{theorem}[A standard sufficient condition for classification calibration]
\label{thm:calibration}
Let $\varphi_1 : \RR \to [0,\infty)$ be a margin-based loss, so that the surrogate risk depends on a real-valued score $t$ only through the margin $y t$. If $\varphi_1$ is convex and differentiable at $0$ with $\varphi_1'(0) < 0$, then $\varphi_1$ is classification calibrated.
\end{theorem}

This theorem is a convenient sufficient condition, but it does not cover every calibrated loss used in practice. In particular, the hinge loss is convex but not differentiable at $0$; nevertheless it is classification calibrated, as shown in \citet{bartlett2006convexity}. In addition to the standard hinge loss $\varphi_{\mathrm{hinge}}(z)$, two specific choices of classification calibrated margin-based loss $\varphi_1$ that we will use later are the least-squares (LS) and the linear-exponential (LINEX) loss functions, resulting in LS-SVM \citep{suykens1999lssvm,suykens2002lssvm} and LINEX-SVM \citep{ma2019linexsvm,tang2021mvlinex}, respectively. For LS-SVM one works with the quadratic margin-based loss
\begin{equation}\label{eq:ls}
  \varphi_{\mathrm{LS}}(z) = (1 - z)^2,
\end{equation}
which is equivalent to the standard LS-SVM formulation with equality-type constraints on the functional margin. To model asymmetric misclassification costs, one can consider the LINEX margin-based loss \citep{zellner1986asym}
\begin{equation}\label{eq:LINEX}
  \varphi_{\mathrm{LINEX}}(z)
  = \exp\!\big(a(1 - z)\big) - a(1 - z) - 1,
  \qquad a \in \RR\setminus\{0\},
\end{equation}
where the parameter $a$ controls the extent and direction of asymmetry in the penalty: $a>0$ puts heavier weight on negative margins than on large positive ones, and $a<0$ reverses this behaviour.\ This choice gives rise to the LINEX-SVM and its extensions  \citep{ma2019linexsvm,tang2021mvlinex}.  

Finally, it is worth mentioning that to capture nonlinear decision boundaries, SVMs are typically combined with feature maps $\Phi : \RR^p \to \mathcal{H}_{\Phi}$ into a (possibly infinite-dimensional) Hilbert space.\ A linear SVM is then fitted in $\mathcal{H}_{\Phi}$ using the transformed data $\Phi(\xx_i)$, but the resulting decision function can be expressed entirely in terms of a positive definite kernel $K(\xx_i,\xx_j) = \langle \Phi(\xx_i), \Phi(\xx_j)\rangle$ via the kernel trick:
\begin{equation}
\label{eq:kernelDecision}
\hat{f}(\xx) = \hat{\beta}_0 + \sum_{j=1}^n \hat{\alpha}_j y_j 
K(\xx_j,\xx).
\end{equation}
Standard choices of kernel functions  include the linear, polynomial, and Gaussian RBF kernels; see, e.g., \citet{scholkopf2002kernels,steinwart2008svm} for comprehensive treatments. In the next section we build on this classical framework to introduce elite-driven SVMs, which explicitly exploit additional information about benchmark models and influential observations.

%
%

\section{From Benchmark Models to Elite-Driven SVMs}
\label{sec:edsvm}

In standard soft-margin SVMs, the fitted decision boundary is governed primarily by observations on or inside the margin. Points far from the separating hyperplane have zero slack and typically exert no direct influence on the fitted classifier, whereas boundary and margin-violating observations determine the local geometry of the separator. We refer to such influential observations as elite observations. The purpose of EDSVM is to allow selected elite observations to carry additional benchmark information. These may be support vectors from previously fitted SVMs, expert-identified borderline cases, or scientifically important observations whose margin behaviour should be preserved in the new classifier. Thus, instead of treating all observations symmetrically or imposing a global penalty on the decision function, EDSVM gives a privileged role to selected observations and their associated benchmark slacks.

A straightforward way to incorporate a benchmark classifier $f^*$ into the learning of a new classifier $f$ is to penalize the discrepancy between their decision values, for example by minimizing an objective of the form
\[
  \omega \,\varphi_1\big(y f(\xx)\big)
  + (1-\omega)\big(f(\xx) - f^*(\xx)\big)^2,
\]
with $\omega \in [0,1]$ tuning the balance between empirical fit and proximity to the benchmark model. Although simple, this strategy has several limitations.\ First, it enforces a global alignment between $f$ and $f^*$: the penalty $(f(\xx) - f^*(\xx))^2$ is applied uniformly over the entire feature space, including regions where the benchmark is known to be suboptimal or where we deliberately wish to deviate from its behaviour.\  As a result, the new classifier may inherit undesirable artifacts of the benchmark model in regions of low interest.\ Second, this formulation does not naturally accommodate multiple benchmarks $f^*_{1},\ldots,f^*_{L}$ in a way that reflects heterogeneous trust across regions or subsets of observations.\ Finally, a global penalty on $f(\xx)$ does not exploit the core insight of margin-based methods: the classifier is determined by points that lie on or inside the margin.\ Penalizing $f(\xx)-f^*(\xx)$ for points deep inside a class region has little statistical justification and blurs the connection between regularization and decision-critical geometry.

For these reasons, we now work not at the level of the raw decision values $f(\xx)$, but at the level of quantities associated with the margin constraints.

\subsection{Slack-based deviation and elite sets}

Consider again the soft-margin SVM formulation \eqref{eq:softmargin}.\  For a  classifier $C_f(\xx) = \mathrm{sign}\big(f(\xx)\big)$ with decision function
$f(\xx) = \beta_0 + \xx^\top \bbeta$ and a training pair $(\xx_i,y_i)$, the primal constraints are
\[
  y_i f(\xx_i) \;\ge\; 1 - \xi_i, 
  \qquad \xi_i \;\ge\; 0.
\]
Because the objective is increasing in each $\xi_i$, at optimality one always chooses the smallest feasible slack, so that
\[
  \xi_i \;=\; \max\{0,\; 1 - y_i f(\xx_i)\}.
\]
If $y_i f(\xx_i) \le 1$ (the margin constraint is active), then $y_i f(\xx_i) = 1 - \xi_i$; if $y_i f(\xx_i) > 1$, then $\xi_i = 0$ and the constraint is strict. In this standard interpretation,
\begin{itemize}
  \item $\xi_i = 0$ and $y_i f(\xx_i) \ge 1$ means the point is correctly classified and lies on or outside the margin;
  \item $0 < \xi_i < 1$ means it is correctly classified but inside the margin;
  \item $\xi_i \ge 1$ means it is misclassified.
\end{itemize}
The nonnegative slack $\xi_i$ therefore quantifies how much the observation deviates from the large-margin ideal.\ For our purposes, it is convenient to work with the associated margin deviation
\[
  \delta_i \;:=\; 1 - y_i f(\xx_i),
\]
which can take positive or negative values.\  The usual SVM slack is then $\xi_i = \max\{0,\delta_i\}$, so $\delta_i$ coincides with $\xi_i$ on points with active constraints and is negative for correctly classified points strictly outside the margin.\ In particular, $\delta_i < 0$ iff $\xi_i = 0$ and $y_i f(\xx_i) > 1$, i.e., the observation is a correctly classified non-support vector lying strictly outside the margin.

Suppose that, in addition to $f$, we also have a benchmark classifier  $f^*(\xx)$, for example coming from a previously fitted SVM under a different  loss function or kernel. We define its margin deviations analogously by $\delta_i^* \;:=\; 1 - y_i f^*(\xx_i)$.\ Since $y_i^2 = 1$, we obtain
\begin{equation}
\label{eq:marginSlackRelation}
\big(f(\xx_i) - f^*(\xx_i)\big)^2
\;=\;
\big(y_i f(\xx_i) - y_i f^*(\xx_i)\big)^2
\;=\;
\big[(1 - \delta_i) - (1 - \delta_i^*)\big]^2
\;=\;
(\delta_i - \delta_i^*)^2.
\end{equation}
Thus, the squared difference between the signed decision values is exactly the squared difference between the deviations from the ideal margin level~$1$.\ The usual SVM slacks for the two classifiers are obtained by truncation, $\xi_i = \max\{0,\delta_i\}$, and  $\xi_i^* = \max\{0,\delta_i^*\}$,  so on points with active constraints the slacks coincide with the deviations.

This slack- and deviation-based view has two key advantages over working directly with $f(\xx_i)$ and $f^*(\xx_i)$.\ First, it yields a localized, margin-based notion of proximity:  $\xi_i$ and $\xi_i^*$ are nonzero exactly for points that lie strictly inside  the margin or are misclassified, that is, for points with $y_i f(\xx_i) < 1$  and $y_i f^*(\xx_i) < 1$, respectively.\  Penalizing discrepancies through $(\xi_i - \xi_i^*)^2$ therefore concentrates the influence of the benchmark on observations where the classification decision is genuinely nontrivial, and largely ignores regions where both classifiers already achieve a comfortable margin.\  In particular, when both $f$ and $f^*$ place $(\xx_i,y_i)$ strictly outside the margin, the corresponding slacks vanish and the penalty contributes nothing on that point. This stands in contrast to global penalties on $(f(\xx_i) - f^*(\xx_i))^2$, which treat all points equally, including those deep in the interior of each class region.\ Second, the deviation representation decouples benchmark guidance from the parametric form of the classifiers.\ The identity \eqref{eq:marginSlackRelation} does not rely on $f$ and $f^*$ sharing the same intercept, weight vector, or even function space: $f$ may be linear in the original input space, while $f^*$ may arise from a kernel SVM in a high-dimensional feature space, or vice versa. 


To formalize the idea that only some observations carry benchmark information, we let $\mathcal{S}^* \subseteq \{1,\ldots,n\}$ denote the set of elite indices.\ These are the training points for which we possess reliable benchmark slacks $\xi_i^*$ and on which we wish to enforce proximity between $\xi_i$ and $\xi_i^*$.\ A natural and practically relevant choice is to take  $\mathcal{S}^*$ as the union of the support vectors from one or more benchmark SVM models.\ In that case, $\mathcal{S}^*$ collects precisely the observations that are influential for at least one of the reference classifiers and therefore constitute the elite set whose margin behaviour we would like to preserve or synthesize in the new model.\ For indices $i \notin \mathcal{S}^*$, no benchmark information is used and the slacks $\xi_i$ are driven entirely by the data and the base loss.

Since the indexing of observations is arbitrary, we can, without loss of generality, reorder the data such that
$$\mathcal{S}^* = \{1,2,\ldots,m\}, \qquad m \le n.$$
This notational simplification allows us to write benchmark-guided penalties as simple sums over the first $m$ indices, which leads to compact expressions for the EDSVM loss and its dual formulation.\ In what follows, we exploit this structure to construct composite objectives that combine empirical risk minimization with elite-driven proximity to one or several benchmark classifiers.

%
%

\section{Elite-driven loss at the slack level}
\label{sec:edsvm-loss}

We now introduce the general elite-driven SVM (EDSVM) loss at the level of the slack variables. Recall that, after reordering, the elite indices are collected in $\mathcal{S}^* = \{1,\ldots,m\}$, and for each $i \in \mathcal{S}^*$ we assume access to a benchmark slack value $\xi_i^*$, computed as in Section~\ref{sec:edsvm} from one or more reference classifiers.\  For non-elite points $i > m$, no benchmark information is used.

Given a base slack loss $\varphi_1$ and a nonnegative distance function $d$ that measures deviation from benchmark behaviour, we define the composite. EDSVM loss
\begin{equation}
\label{eq:edsvmLoss}
\varphi_\omega(\xi_i) =
\begin{cases}
\omega \,\varphi_1(\xi_i) + (1-\omega)\, d(\xi_i,\xi_i^*), & i \le m, \\[0.25em]
\omega \,\varphi_1(\xi_i), & i > m,
\end{cases}
\end{equation}
where $\omega \in [0,1]$ is a tuning parameter.\ The function $\varphi_1$ is typically induced by a margin-based surrogate loss used in the underlying SVM formulation (for instance, hinge or squared-slack loss), and therefore embodies the usual empirical fit and margin properties. The term $d(\xi_i,\xi_i^*)$ penalizes departures of the new slack $\xi_i$ from its benchmark counterpart $\xi_i^*$, and is chosen so that $d(\xi_i,\xi_i^*) \ge 0$ with equality if and only if $\xi_i = \xi_i^*$.\  A natural example is the squared Euclidean distance $d(\xi_i,\xi_i^*) = (\xi_i-\xi_i^*)^2$, but other convex distances (absolute deviation, Huber-type distances) could be used to trade off fidelity and robustness.

The parameter $\omega$ regulates the relative weight given to the data-driven term and the benchmark-guided term on the elite set. When $\omega = 1$, the EDSVM loss reduces to the usual slack-based loss and the resulting estimator coincides with the standard SVM-type model associated with $\varphi_1$.\  At the opposite extreme, when $\omega = 0$, the slacks of elite points are encouraged to reproduce the benchmark slacks as closely as permitted by the constraints, so that the new decision boundary is essentially anchored to the margin behaviour of the reference classifier(s) on $\mathcal{S}^*$.\ For intermediate values $0 < \omega < 1$, we obtain a continuum of models that interpolate between pure empirical risk minimization and pure benchmark imitation: small values of $\omega$ induce strong shrinkage of $\xi_i$ towards $\xi_i^*$ on the elite set, while larger values prioritize the base loss and allow the new classifier to deviate more freely from the benchmark.\ Importantly, for non-elite points $i > m$ the penalty reduces to $\omega \varphi_1(\xi_i)$, so that benchmark information is injected only where it is available and deemed relevant.

Embedding the composite loss \eqref{eq:edsvmLoss} into the soft-margin SVM framework leads to the general EDSVM primal problem
\begin{equation}
\label{eq:edsvmERM}
\min_{\beta_0,\bbeta,\bxi}
~\frac{1}{2} \|\bbeta\|_2^2
+ C \sum_{i=1}^n \varphi_\omega(\xi_i),
\quad\text{subject to}\quad
y_i(\beta_0 + \xx_i^\top \bbeta) \ge 1 - \xi_i,\;
\xi_i \ge 0,\quad i=1,\ldots,n.
\end{equation}
Here $C>0$ plays the familiar role of a global regularization parameter.\ Provided that both $\varphi_1(\cdot)$ and $d(\cdot,\xi_i^*)$ are convex in their first argument, the composite loss $\varphi_\omega$ is convex for each fixed $\omega \in [0,1]$, and the optimization problem \eqref{eq:edsvmERM} remains a convex program, amenable to standard SVM solvers. From a geometric perspective, the regularization term $\frac{1}{2}\|\bbeta\|_2^2$ continues to control margin width, while the modified slack penalty reshapes the way in which margin violations are penalized on elite versus non-elite points.

The formulation in \eqref{eq:edsvmLoss} is also flexible enough to accommodate multiple benchmark models. Suppose that, for each observation $i$ in the elite set, we have $L$ benchmark classifiers, giving rise to a vector of benchmark slacks
\[
  \big(\xi_{i,1},\ldots,\xi_{i,L}\big),
\]
for example obtained from SVMs fitted with different kernels, loss functions, or tuning parameters. We then define a single target slack
\begin{equation*}
  \xi_i^* = g(\xi_{i,1},\ldots,\xi_{i,L}),
\end{equation*}
where $g$ is an aggregation functional chosen to reflect how the different benchmarks should be combined.\ Taking $g$ as the mean produces a compromise that pulls the new model towards the average margin behaviour across benchmarks; using the maximum or minimum yields conservative or optimistic alignment; median or trimmed means offer robustness when one of the reference classifiers is suspected to be misspecified.\  More elaborate choices of $g$ can encode class-specific or group-specific trust in particular benchmarks, or emphasize certain regions of the feature space through observation-dependent weights.\ In all cases, the resulting $\xi_i^*$ enters the EDSVM objective only through the distance term $d(\xi_i,\xi_i^*)$ on the elite indices, ensuring that benchmark guidance remains localized and interpretable at the slack level.

In the next subsections we further develop this general framework with two specific choices for the base loss and distance function, leading to the C-EDSVM (hinge-based) and LS-EDSVM (squared-slack-based) formulations, and derive their corresponding dual problems.

%
%

\subsection{C-EDSVM: Hinge loss with slack deviation penalty}

A natural first instance of the EDSVM framework is obtained by choosing
\begin{equation*}
\varphi_1(\xi_i) = \xi_i, \qquad 
d(\xi_i,\xi_i^*) = (\xi_i - \xi_i^*)^2 \quad (i \le m).
\end{equation*}
This yields the C-EDSVM objective
\begin{equation}
\label{eq:cedsvmPrimal}
\begin{aligned}
\min_{\beta_0,\bbeta,\bxi} \quad 
& \frac{1}{2} \|\bbeta\|_2^2 
+ C \Bigg[ \omega \sum_{i=1}^n \xi_i 
+ (1-\omega) \sum_{i=1}^m (\xi_i - \xi_i^*)^2 \Bigg] \\
\text{subject to} \quad 
& y_i(\beta_0 + \xx_i^\top \bbeta) \ge 1 - \xi_i, ~\text{and}~  \xi_i \ge 0, \quad i=1,\ldots,n.
\end{aligned}
\end{equation}
We refer to this model as C-EDSVM because it augments the standard C-SVM hinge-loss formulation with a quadratic penalty on deviations from benchmark slacks on elite points.\ To analyze the dual, consider the Lagrangian
\begin{equation*}
\begin{aligned}
\mathcal{L}(\beta_0,\bbeta,\bxi,\bm{\alpha},\bm{\lambda})
&= \frac{1}{2} \|\bbeta\|_2^2 
+ C \Bigg[ \omega \sum_{i=1}^n \xi_i 
+ (1-\omega) \sum_{i=1}^m (\xi_i - \xi_i^*)^2 \Bigg] \\
&\quad - \sum_{i=1}^n \alpha_i \big[ y_i(\beta_0 + \xx_i^\top \bbeta) - 1 + \xi_i \big]
- \sum_{i=1}^n \lambda_i \xi_i,
\end{aligned}
\end{equation*}
with multipliers $\alpha_i \ge 0$ for the margin constraints and $\lambda_i \ge 0$ for $\xi_i \ge 0$.\ Stationarity with respect to $\beta_0$ and $\bbeta$ leads to
\begin{equation}
\label{eq:cedsvmKKT1}
\sum_{i=1}^n \alpha_i y_i = 0,
\qquad
\bbeta^* = \sum_{i=1}^n \alpha_i y_i \xx_i.
\end{equation}
Differentiating with respect to $\xi_i$ gives, for $i \le m$,
\begin{equation}
\label{eq:cedsvmXiElite}
  C\omega + 2C(1-\omega)(\xi_i - \xi_i^*) - \alpha_i - \lambda_i = 0,
\end{equation}
and for $i > m$,
\begin{equation}
\label{eq:cedsvmXiNonElite}
  C\omega - \alpha_i - \lambda_i = 0.
\end{equation}
Solving \eqref{eq:cedsvmXiElite} for $\xi_i$ yields
\begin{equation}
\label{eq:cedsvmXiSolution}
\xi_i = \xi_i^* + \frac{1}{2(1-\omega)}\left(\frac{\alpha_i + \lambda_i}{C}
- \omega\right), \qquad i = 1,\ldots,m,
\end{equation}
while \eqref{eq:cedsvmXiNonElite} together with $\lambda_i\ge 0$ implies
$0\le \alpha_i \le C\omega$ for $i > m$.\ Substituting these expressions back into the Lagrangian and eliminating the
primal variables yields a dual problem of the form
\begin{equation}
\label{eq:cedsvmDual}
\begin{aligned}
\max_{\bm{\alpha}} \quad  
-\frac{1}{2}\bm{\alpha}^\top Q \bm{\alpha}
+ \bm{R}^\top \bm{\alpha} + D \quad 
\text{subject to} \quad 
&\bm{\alpha}^\top \bm{y} = 0,\\ \quad 
&\alpha_i \ge 0,\; i=1,\ldots,m, \\\quad
&0\le \alpha_i \le C\omega,\; i=m+1,\ldots,n,
\end{aligned}
\end{equation}
where $\bm{y} = (y_1,\ldots,y_n)^\top$, 
\begin{equation}
\label{eq:cedsvmQ}
  Q = H + \mathrm{diag}(\bm{d}), 
  \qquad
  H_{ij} = y_i y_j \xx_i^\top \xx_j,
\end{equation}
\begin{equation}
\label{eq:cedsvmDvec}
  d_i = 
  \begin{cases}
  \dfrac{1}{2C(1-\omega)}, & i \le m,\\[0.4em]
  0, & i > m,
  \end{cases}
\quad\text{and} \quad 
  R_i =
  \begin{cases}
  1 - \xi_i^* - \dfrac{\omega}{2(1-\omega)}, & i \le m,\\[0.4em]
  1, & i > m,
  \end{cases}
\end{equation}
and
\begin{equation}
\label{eq:cedsvmDconst}
  D = -\frac{m C \omega^2}{4(1-\omega)} + C\omega \sum_{i=1}^m \xi_i^*.
\end{equation}
To verify: substituting \eqref{eq:cedsvmXiSolution} into the elite-point terms of the Lagrangian, writing $s_i:=\alpha_i+\lambda_i-C\omega$, the constant (terms free of primal and dual variables) from point $i\le m$ is
\[
  C\omega\xi_i^* + \frac{\omega s_i}{2(1-\omega)} + \frac{s_i^2}{4C(1-\omega)}
  - (\alpha_i+\lambda_i)\xi_i^* - \frac{(\alpha_i+\lambda_i)s_i}{2C(1-\omega)}
  \xrightarrow{\text{set }s_i=0}
  -\frac{C\omega^2}{4(1-\omega)},
\]
summed over $i=1,\ldots,m$ gives $D = -\frac{mC\omega^2}{4(1-\omega)} + C\omega\sum_{i=1}^m\xi_i^*$.
This does not affect the optimizer but is required for correct evaluation of the dual objective.\  The box constraint $0\le\alpha_i\le C\omega$ for $i>m$ follows directly from the KKT condition \eqref{eq:cedsvmXiNonElite}: since $\lambda_i\ge0$, we have $\alpha_i = C\omega-\lambda_i\le C\omega$; and since $C\omega - \alpha_i = \lambda_i \ge 0$, we also have $\alpha_i \le C\omega$. Non-negativity $\alpha_i\ge0$ follows from the dual constraint directly.\ Elite points carry no finite upper bound because their stationarity condition \eqref{eq:cedsvmXiElite} does not pin $\alpha_i$ to a fixed value.\ Strong duality holds by Slater's condition (the primal is strictly feasible for any finite $\xi_i^*$), so the two problems are equivalent.\ The elite guidance appears in the dual through both the diagonal regularization $\mathrm{diag}(\bm{d})$ and the linear term $\bm{R}$.\ In a kernelized version, $H_{ij}$ in \eqref{eq:cedsvmQ} is replaced by $y_i y_j K(\xx_i,\xx_j)$ and the structure of the dual remains unchanged.

We now identify the margin-based loss implicitly induced by the C-EDSVM primal formulation in \eqref{eq:cedsvmPrimal} and give simple sufficient conditions under which it satisfies Theorem~\ref{thm:calibration}.\ Consider a real-valued decision function $f:\RR^p\to\RR$ with margin $u = y f(\xx)$. Ignoring the regularizer $\tfrac12\|\bbeta\|_2^2$ and the overall factor $C>0$, the per-example surrogate contribution in
\eqref{eq:cedsvmPrimal} is
\[
  \ell_i\bigl(y_i,f(\xx_i),\xi_i\bigr)
  \;=\;
  \omega\,\xi_i
  + (1-\omega)\,\1\{i\le m\}\,(\xi_i - \xi_i^*)^2,
\]
subject to $\xi_i \ge 0$ and $\xi_i \ge 1-u$.\ For a fixed pair $(\xx,y)$ and margin $u = y f(\xx)$, the induced margin loss at index $i$ is obtained by minimizing over feasible slacks:
\[
  \phi_i(u)
  \;:=\;
  \min_{\xi\ge 0}
  \Bigl\{\,
    \omega\,\xi
    + (1-\omega)\,\1\{i\le m\}\,(\xi-\xi_i^*)^2
    :\; \xi \ge 1-u
  \Bigr\}.
\]
For non-elite points ($i>m$), the quadratic term is absent and the constraint $\xi\ge 1-u$ immediately yields
\[
  \phi_i(u) \;=\; \omega (1-u)_+,
  \qquad i>m,
\]
that is, a positive scaling of the standard hinge loss.\ For elite points ($i\le m$), define
\[
  g_i(\xi) \;:=\; \omega\,\xi + (1-\omega)(\xi-\xi_i^*)^2,
  \qquad
  \bar\xi_i \;:=\; \xi_i^* - \frac{\omega}{2(1-\omega)}.
\]
The function $g_i$ is strictly convex in $\xi$ with unconstrained minimizer $\bar\xi_i$. Incorporating the feasibility constraints $\xi\ge 0$ and $\xi\ge 1-u$ therefore gives
\[
  \xi_i(u) \;=\; \max\bigl\{\,0,\;1-u,\;\bar\xi_i\,\bigr\},
  \qquad
  \phi_i(u) \;=\; g_i\bigl(\xi_i(u)\bigr),\qquad i\le m.
\]
Convexity of $\phi_i$ in $u$ follows from a two-step composition argument. {Step 1.} The inner function $h(u)=\max\{0,1-u,\bar\xi_i\}$ is convex because it is the pointwise maximum of three convex functions (two constants and the linear function $1-u$); pointwise maxima of convex functions are convex. Moreover, $h(u)\ge\bar\xi_i$ for all $u$ by definition of the maximum, so the range of $h$ is contained in $[\bar\xi_i,\infty)$. {Step 2.} The penalty $g_i(\xi)=\omega\xi+(1-\omega)(\xi-\xi_i^*)^2$ is strictly convex and nondecreasing on $[\bar\xi_i,\infty)$ (since $g_i'(\xi)=\omega+2(1-\omega)(\xi-\xi_i^*)\ge0$ for all $\xi\ge\bar\xi_i$, with equality only at $\xi=\bar\xi_i$); since the range of $h$ lies in $[\bar\xi_i,\infty)$ where $g_i$ is nondecreasing, the standard monotone-composition rule applies: $\phi_i=g_i\circ h$ is convex in $u$. 

To apply Theorem~\ref{thm:calibration}, we examine $\phi_i'(0)$. Suppose first that, for some neighborhood of $u=0$, the active constraint is $\xi_i(u)=1-u$. This holds whenever
\[
\bar\xi_i < 1
\quad\Longleftrightarrow\quad
\xi_i^* < 1 + \frac{\omega}{2(1-\omega)}.
\]
%
Under this condition we have, for $u$ in a neighborhood of $0$,
\[
  \phi_i(u)
  = \omega(1-u)
    + (1-\omega)\bigl((1-u)-\xi_i^*\bigr)^2,
\]
and a direct calculation yields
\[
  \phi_i'(u)
  = -\omega + (1-\omega)\bigl(2u + 2\xi_i^* - 2\bigr)
  \qquad
  \Rightarrow\qquad
  \phi_i'(0)
  = -\omega + (1-\omega)\bigl(2\xi_i^* - 2\bigr).
\]
If $\xi_i^*\le 1$, then $\phi_i'(0) \;\le\; -\omega \;<\; 0.$ More generally,
the inequality
\[
  \xi_i^*
  \;<\;
  1 + \frac{\omega}{2(1-\omega)}
\]
is sufficient to ensure $\phi_i'(0)<0$.\ We can now summarize the calibration properties of the C-EDSVM surrogate.

\begin{proposition}
\label{prop:cedsvm-calibration}
Assume $0<\omega<1$ and that all elite benchmark slacks satisfy
\[
  \xi_i^* \;<\; 1 + \frac{\omega}{2(1-\omega)},
  \qquad i=1,\ldots,m.
\]
Then, for every $i$, the induced margin loss $\phi_i$ is convex and differentiable at $0$ with $\phi_i'(0)<0$. In particular, the C-EDSVM surrogate loss satisfies the sufficient condition of Theorem~\ref{thm:calibration} and is therefore classification calibrated.
\end{proposition}

\begin{remark}
\label{rem:heterogeneous-family}
The EDSVM induces a heterogeneous family of margin losses $\{\phi_i\}$ indexed by the benchmark slacks $\xi_i^*$, rather than a single homogeneous loss $\phi$. The calibration argument is applied pointwise so that for each $i$, the induced loss $\phi_i$ satisfies the differentiability and negativity conditions of Theorem~\ref{thm:calibration}. This is a valid application because the theorem requires only that the per-example loss used in minimization be classification calibrated; here, different elite observations are governed by different but individually calibrated $\phi_i$. Non-elite observations use the homogeneous hinge loss scaled by $\omega$, which is also calibrated. The heterogeneity of $\phi_i$ across elite points does not invalidate the argument, but it should be borne in mind when interpreting the conditions that the bound $\xi_i^* < 1 + \omega/(2(1-\omega))$ must hold for each elite point separately.
\end{remark}

\begin{remark}
\label{rem:xi-bounded}
A particularly transparent sufficient condition is $\xi_i^*\le 1$ for all elite points. For example, if elites are chosen among training points with small hinge slack under the baseline SVM used to generate $\xi_i^*$, then $\xi_i^*\in[0,1]$ typically holds. In that case $\phi_i'(0)\le -\omega < 0$ for every $i\le m$, so C-EDSVM is classification calibrated for any $\omega\in(0,1)$.
\end{remark}

%
%

\subsection{LS-EDSVM: least-squares SVM with slack guidance}

We now consider a formulation inspired by least-squares SVM (LS-SVM), in which 
the slack variables enter quadratically. Let
\begin{equation*}
\varphi_1(\xi_i) = \xi_i^2,
\qquad
d(\xi_i,\xi_i^*) = (\xi_i - \xi_i^*)^2 \quad (i \le m).
\end{equation*}
The LS-EDSVM primal then becomes
\begin{equation}
\label{eq:lSEDPrimal}
\begin{aligned}
\min_{\beta_0,\bbeta,\bxi} \quad
& \frac{1}{2}\|\bbeta\|_2^2 
+ C\left[\omega \sum_{i=1}^n \xi_i^2 
+ (1-\omega) \sum_{i=1}^m (\xi_i - \xi_i^*)^2\right] \\
\text{subject to} \quad 
& y_i(\beta_0 + \xx_i^\top \bbeta) \ge 1 - \xi_i, ~\text{and}~ \xi_i \ge 0, \quad i=1,\ldots,n.
\end{aligned}
\end{equation}
The Lagrangian is
\begin{equation}
\begin{aligned}
\mathcal{L}(\beta_0,\bbeta,\bxi,\bm{\alpha},\bm{\lambda})
&= \frac{1}{2} \|\bbeta\|_2^2 
+ C \sum_{i=1}^n \omega \xi_i^2
+ C \sum_{i=1}^m (1-\omega)(\xi_i - \xi_i^*)^2 \\
&\quad - \sum_{i=1}^n \alpha_i \big[ y_i(\beta_0 + \xx_i^\top \bbeta) - 1 + \xi_i \big]
- \sum_{i=1}^n \lambda_i \xi_i.
\end{aligned}
\end{equation}
Stationarity with respect to $\beta_0$ and $\bbeta$ yields
\begin{equation*}
\sum_{i=1}^n \alpha_i y_i = 0,
\qquad
\bbeta^* = \sum_{i=1}^n \alpha_i y_i \xx_i.
\end{equation*}
Differentiating with respect to $\xi_i$ gives, for $i \le m$,
\begin{equation}
  2C\left[\omega \xi_i + (1-\omega)(\xi_i - \xi_i^*)\right] 
  - \alpha_i - \lambda_i = 0,
\end{equation}
and for $i > m$,
\begin{equation}
  2C\omega \xi_i - \alpha_i - \lambda_i = 0.
\end{equation}
Solving for $\xi_i$ results in
\begin{equation}
\label{eq:lSEDXiSolution}
\xi_i =
\begin{cases}
\dfrac{\alpha_i + \lambda_i}{2C} + (1-\omega)\xi_i^*, & i \le m, \\[0.5em]
\dfrac{\alpha_i + \lambda_i}{2C\omega}, & i > m.
\end{cases}
\end{equation}
Substituting into the Lagrangian and eliminating the primal variables leads to  a dual problem of the same structure as \eqref{eq:cedsvmDual},
\begin{equation}
\max_{\bm{\alpha}} ~-\frac{1}{2}\bm{\alpha}^\top Q \bm{\alpha} + \bm{R}^\top \bm{\alpha} + D \quad 
\text{subject to} \quad 
\bm{\alpha}^\top \bm{y} = 0,\quad 
\alpha_i \ge 0,\; i=1, \ldots, n,
\end{equation}
where now
\begin{equation}
\label{eq:lSEDQ}
  Q = H + \mathrm{diag}(\bm{d}), 
  \qquad
  H_{ij} = y_i y_j \xx_i^\top \xx_j,
\end{equation}
and 
\begin{equation}
\label{eq:lSEDDvec}
  d_i = 
  \begin{cases}
  \dfrac{1}{2C}, & i \le m,\\[0.4em]
  \dfrac{1}{2C\omega}, & i > m,
  \end{cases}
\quad\text{and}\quad 
  R_i =
  \begin{cases}
  1 - (1-\omega)\xi_i^*, & i \le m,\\[0.4em]
  1, & i > m.
  \end{cases}
\end{equation}
Also, 
\begin{equation}
\label{eq:lSEDDconst}
  D = C(1-\omega) \sum_{i=1}^m (\xi_i^*)^2.
\end{equation}
As in the C-EDSVM case, the elite guidance enters through a modified diagonal regularization and a shifted linear term. The kernelized version proceeds by replacing inner products with $K(\xx_i,\xx_j)$.

We now identify the margin-based loss implicitly induced by the LS-EDSVM primal formulation in \eqref{eq:lSEDPrimal} and give sufficient conditions under which it satisfies Theorem~\ref{thm:calibration}.\ The per-example contribution to the surrogate loss is 
\[
  \ell_i\bigl(y_i,f(\xx_i),\xi_i\bigr)
  \;=\;
  \omega\,\xi_i^2
  + (1-\omega)\,\1\{i\le m\}\,(\xi_i - \xi_i^*)^2,
\]
subject to $\xi_i \ge 0$ and $\xi_i \ge 1-u_i$ with $u_i = y_i f(\xx_i)$. For fixed $(\xx,y)$ and margin $u$, the induced margin loss at index $i$ is
\[
  \phi_i(u)
  \;:=\;
  \min_{\xi\ge 0}
  \Bigl\{\,
    \omega\,\xi^2
    + (1-\omega)\,\1\{i\le m\}\,(\xi-\xi_i^*)^2
    :\; \xi \ge 1-u
  \Bigr\}.
\]
For non-elite points ($i>m$), the guidance term is absent, and the objective simplifies to $\omega\xi^2$ with unconstrained minimizer $\xi=0$.\ The constraints then give
\[
  \xi_i(u) \;=\; \max\{0,1-u\},
  \qquad
  \phi_i(u)
  \;=\;
  \omega\bigl(\max\{0,1-u\}\bigr)^2,
  \qquad i>m.
\]
In particular, for $u$ in a neighborhood of $0$ we have $1-u>0$ and hence
$\xi_i(u)=1-u$, so that
\[
  \phi_i(u) = \omega(1-u)^2
  \quad\Rightarrow\quad
    \phi_i'(0) = -2\omega < 0.
\]
For elite points ($i\le m$), define
\[
  g_i(\xi) \;:=\; \omega\,\xi^2 + (1-\omega)(\xi-\xi_i^*)^2,
  \qquad
  \bar\xi_i \;:=\; (1-\omega)\,\xi_i^*.
\]
The function $g_i$ is strictly convex in $\xi$ with unconstrained minimizer $\bar\xi_i$. To apply Theorem~\ref{thm:calibration}, we analyze $\phi_i'(0)$.\  There are two regimes, depending on the relative size of $\bar\xi_i$ and $1$.\ If $\bar\xi_i < 1$, then for $u$ in a neighborhood of $0$ we have $1-u > \bar\xi_i$ and $1-u>0$, so that
\[
  \xi_i(u) = 1-u,
  \qquad
  \phi_i(u) = g_i(1-u).
\]
Using $g_i'(\xi) =  2(\xi - \bar\xi_i)$, we obtain
\[
  \phi_i'(u)
    = 2\bigl[\bar\xi_i - 1 + u\bigr],
\]
and hence
\[
  \phi_i'(0) = 2(\bar\xi_i - 1)
  = 2\bigl((1-\omega)\xi_i^* - 1\bigr).
\]
The condition $\bar\xi_i < 1$ is equivalent to
\[
  (1-\omega)\xi_i^* < 1
  \quad\Longleftrightarrow\quad
  \xi_i^* < \frac{1}{1-\omega},
\]
and under this condition we indeed have $\phi_i'(0)<0$.\ If $\bar\xi_i > 1$, then for $u$ sufficiently close to $0$ we have $1-u < \bar\xi_i$ and $0 \le 1-u$, so that
\[
  \xi_i(u) = \bar\xi_i,
  \qquad
  \phi_i(u) = g_i(\bar\xi_i)
\]
is locally constant in $u$.\  In this regime, $\phi_i'(0) = 0$ and the sufficient condition from Theorem~\ref{thm:calibration} fails.\  These observations yield the following sufficient condition for classification calibration of LS-EDSVM.

\begin{proposition}
\label{prop:lSED-calibration}
Assume $0<\omega<1$ and that all elite benchmark slacks satisfy
\[
  \xi_i^* \;<\; \frac{1}{1-\omega},
  \qquad i=1,\ldots,m.
\]
Then, for every $i$, the induced margin loss $\phi_i$ is convex and differentiable at $0$ with $\phi_i'(0)<0$.\ In particular, the LS-EDSVM surrogate loss satisfies the sufficient condition of Theorem~\ref{thm:calibration} and is therefore classification calibrated.
\end{proposition}

\begin{remark}
\label{rem:lSED-xi-bounded}
A particularly simple sufficient condition is $\xi_i^*\le 1$ for all elite points.\ In that case,
\[
  \phi_i'(0)
  = 2\bigl((1-\omega)\xi_i^* - 1\bigr)
  \;\le\;
   -2\omega < 0,
\]
so LS-EDSVM is classification calibrated for any $\omega\in(0,1)$.
\end{remark}

%
%

\section{Generalization Analysis of EDSVM}
\label{sec:theory}

The classification calibration results of Propositions~\ref{prop:cedsvm-calibration} and~\ref{prop:lSED-calibration} establish that the EDSVM surrogates are aligned with the Bayes rule in the population limit. The remaining questions are how benchmark quality enters the complexity bounds, when benchmark guidance is beneficial, and whether localization yields rates sharper than the global $n^{-1/2}$ bound. The distinguishing feature of EDSVM is that the empirical criterion contains an additional benchmark-guided term restricted to the elite set, and this alters the form of the natural oracle comparison. As a consequence, the complexity bounds depend explicitly on benchmark fidelity and on the guidance weight $\omega$, thereby exhibiting a trade-off between approximation to the benchmark and estimation error. In this section we develop this analysis in a unified notation that applies to both the hinge-based and least-squares formulations. Throughout, uppercase letters denote random variables and lowercase letters denote realized observations.

Let $\{(\XX_i,Y_i)\}_{i=1}^n$ be i.i.d.\ draws from a distribution $\PP$ on $\mathcal X\times\{-1,+1\}$, with $\mathcal X\subset\mathbb R^p$ compact. Let $K:\mathcal X\times\mathcal X\to\mathbb R$ be a continuous positive-definite kernel with RKHS $\mathcal H_K$, inner product $\langle\cdot,\cdot\rangle_K$, and norm $\|\cdot\|_K$. We assume
\[
K(\xx,\xx)\le \kappa^2,\qquad \xx\in\mathcal X.
\]
A generic decision function is written as
\[
f(\xx)=\beta_0+h(\xx),\qquad \beta_0\in\mathbb R,\quad h\in\mathcal H_K,
\]
and the regularization penalty is $\frac12\|h\|_K^2$, so only the RKHS component is penalized and the intercept remains unpenalized, as in the standard SVM literature. For empirical minimizers, the extended representer theorem yields the sample-dependent representation
\[
f(\xx)=\beta_0+\sum_{j=1}^n \alpha_jY_jK(\XX_j,\xx),
\]
where $\beta_0$ is determined by the KKT complementary-slackness conditions. We use the shorthand $\|f\|_K:=\|h\|_K$ to denote the RKHS norm of the kernel component of $f=\beta_0+h$.

Because the intercept is not penalized, RKHS norm control alone does not bound the full margin $f(\xx)$. To keep the empirical-process arguments rigorous while preserving the usual unpenalized-intercept formulation, we work with a deterministic intercept envelope: throughout this section we assume there exists a finite constant $M_0>0$ such that every comparator and estimator under consideration satisfies $|\beta_0|\le M_0$. This assumption is used only to control margins and loss envelopes; it does not add any penalty to the optimization problem. For $B>0$, define the working class
\[
\mathcal F(B,M_0)=\{\beta_0+h: |\beta_0|\le M_0,\ \|h\|_K\le B\}.
\]
Then, by the reproducing property,
\[
|f(\xx)|\le M_0+\kappa B,
\qquad f\in\mathcal F(B,M_0),\ \xx\in\mathcal X.
\]

Fix a realized sample $\{(\xx_i,y_i)\}_{i=1}^n$. For the hinge loss $\phi(z)=(1-z)_+$, define the population and empirical surrogate risks
\[
R_\phi(f)=\E\{\phi(Yf(\XX))\},\qquad
\hat R_\phi(f)=\frac1n\sum_{i=1}^n \phi\!\bigl(y_i f(\xx_i)\bigr),
\]
and let
\[
f_\phi^*=\beta_{0,\phi}^*+h_\phi^*\in\argmin_f R_\phi(f)
\]
be any fixed population comparator in the class under consideration. The $0$--$1$ risk is
\[
R_{0\text{-}1}(f)=\PP\!\bigl(Y\neq \mathrm{sign}(f(\XX))\bigr),\qquad
R^*=\inf_f R_{0\text{-}1}(f).
\]

For the realized sample, define the slack variables
\[
\xi_i(f)=(1-y_if(\xx_i))_+,\qquad i=1,\dots,n.
\]
The empirical C-EDSVM risk is
\begin{equation}
\label{eq:empirical-edsvm}
\hat R_n^{\mathrm{ED}}(f)
=
\omega\,\hat R_\phi(f)
+
\frac{1-\omega}{n}\sum_{i=1}^m \bigl(\xi_i(f)-\xi_i^*\bigr)^2,
\qquad \omega\in(0,1].
\end{equation}
Here $\{1,\dots,m\}$ indexes the elite subset after relabelling. The benchmark-quality functional is
\begin{equation}
\label{eq:benchqual}
\mathcal E_m^*
=
\frac1m \sum_{i=1}^m \bigl(\xi_i(f_\phi^*)-\xi_i^*\bigr)^2.
\end{equation}

The quantity $\mathcal E_m^*$ measures the fidelity of the supplied benchmark slacks relative to the population hinge-risk minimizer on the elite set only. The idealized case $\mathcal E_m^*=0$ can occur, for example, when the benchmark is generated from the same population comparator and the elite set is chosen so that the benchmark slacks agree with $\xi_i(f_\phi^*)$ on those points. In realistic noisy problems one should expect $\mathcal E_m^*>0$; the role of the theory is precisely to quantify how much benchmark mismatch can be tolerated.

For a class $\mathcal G$ of real-valued functions on $\mathcal X$, the empirical Rademacher complexity is
\[
\hat{\mathfrak R}_n(\mathcal G)
=
\mathbb E_{\bm\sigma}\!\left[
\sup_{g\in\mathcal G}\frac1n\sum_{i=1}^n \sigma_i g(\xx_i)
\;\middle|\;\xx_1,\dots,\xx_n
\right],
\]
where $\sigma_1,\dots,\sigma_n$ are i.i.d.\ Rademacher signs. For an RKHS ball of kernel components,
\[
\mathcal H_K(B)=\{h\in\mathcal H_K:\|h\|_K\le B\},
\]
the standard estimate is
\begin{equation}
\label{eq:radball}
\hat{\mathfrak R}_n\bigl(\mathcal H_K(B)\bigr)\le \frac{\kappa B}{\sqrt n}.
\end{equation}
For the augmented class with bounded intercept,
\begin{equation}
\label{eq:radball-intercept}
\hat{\mathfrak R}_n\bigl(\mathcal F(B,M_0)\bigr)
\le
\frac{\kappa B+M_0}{\sqrt n}.
\end{equation}
To see this, write $f=\beta_0+h$ and use the sub-additivity of Rademacher complexity: $\hat{\mathfrak R}_n(\mathcal F(B,M_0))\le\hat{\mathfrak R}_n(\mathcal H_K(B))+\hat{\mathfrak R}_n(\{c:\,|c|\le M_0\})$. The first term is bounded by $\kappa B/\sqrt{n}$ from \eqref{eq:radball}. For the second, the class of constant functions $\{c:|c|\le M_0\}$ has Rademacher complexity $M_0\,\E|\bar\sigma_n|\le M_0/\sqrt{n}$, where $\bar\sigma_n=n^{-1}\sum_{i=1}^n\sigma_i$ and the bound follows from Jensen's inequality: $\E|\bar\sigma_n|\le(\E\bar\sigma_n^2)^{1/2}=n^{-1/2}$. Adding gives \eqref{eq:radball-intercept}. Bounds of the form \eqref{eq:radball}--\eqref{eq:radball-intercept} follow from the reproducing property and standard symmetrization arguments; see, e.g., \citet{bartlett2002rademacher} and \citet{ledoux1991probability}.

\subsection{A benchmark-dependent norm bound}
\label{sec:normbound}

The first result identifies the effective radius of the class explored by the C-EDSVM estimator.

\begin{lemma}[Benchmark-dependent radius for C-EDSVM]
\label{lem:normbound}
Let $\hat f$ be any minimizer of the C-EDSVM primal \eqref{eq:cedsvmPrimal}. Then, for every comparator $f_0\in\mathcal F$,
\begin{equation}
\label{eq:normbound}
\frac{1}{2n}\|\hat f\|_K^2
\le
\frac{1}{2n}\|f_0\|_K^2
+
C\,\hat R_n^{\mathrm{ED}}(f_0).
\end{equation}
In particular, with $f_0=f_\phi^*$,
\begin{equation}
\label{eq:normboundstar}
\|\hat f\|_K^2
\le
\Lambda_n^2
:=
\|f_\phi^*\|_K^2
+
2Cn\omega\,\hat R_\phi(f_\phi^*)
+
2Cm(1-\omega)\mathcal E_m^*.
\end{equation}
\end{lemma}

\begin{proof}
Optimality of $\hat f$ gives
\[
\frac12\|\hat f\|_K^2 + Cn\,\hat R_n^{\mathrm{ED}}(\hat f)
\le
\frac12\|f_0\|_K^2 + Cn\,\hat R_n^{\mathrm{ED}}(f_0).
\]
Dropping the nonnegative term $Cn\,\hat R_n^{\mathrm{ED}}(\hat f)$ yields \eqref{eq:normbound}. Substituting $f_0=f_\phi^*$ and expanding \eqref{eq:empirical-edsvm} gives \eqref{eq:normboundstar}. \qedhere
\end{proof}

The bound \eqref{eq:normboundstar} makes the bias--variance mechanism explicit. Relative to standard C-SVM, the empirical-risk contribution is shrunk by the factor $\omega$, while benchmark mismatch contributes the additive term $2Cm(1-\omega)\mathcal E_m^*$. Thus the benchmark is helpful when the reduction in effective radius outweighs the benchmark-induced bias. This is precisely the regime where EDSVM can improve on classical SVM bounds.

For comparison, the standard C-SVM minimizer $\hat f_{\mathrm{SVM}}$ satisfies
\begin{equation}
\label{eq:normboundSVM}
\|\hat f_{\mathrm{SVM}}\|_K^2
\le
\|f_\phi^*\|_K^2 + 2Cn\,\hat R_\phi(f_\phi^*).
\end{equation}

\begin{corollary}[Deterministic high-probability radius]
\label{cor:det-radius}
Assume the hinge loss is bounded by $B_\phi$ on the range of margins achieved by $f_\phi^*$. Under the intercept envelope and the reproducing property, one may take
\[
B_\phi = 1+|\beta_{0,\phi}^*|+\kappa\|h_\phi^*\|_K \le 1+M_0+\kappa\|h_\phi^*\|_K,
\]
which is finite on the compact input space. Let $\bar{\mathcal{E}}_m$ be any nonrandom constant satisfying $\mathcal E_m^*\le \bar{\mathcal{E}}_m$ almost surely; for example, since $\xi_i(f_\phi^*)\le B_\phi$ and the benchmark slacks $\xi_i^*$ are fixed, one may take $\bar{\mathcal{E}}_m=\frac{1}{m}\sum_{i=1}^m(B_\phi+\xi_i^*)^2$. Define
\begin{equation}
\label{eq:det-radius}
\Lambda_{n,\delta}^{\circ\,2}
:=
\|f_\phi^*\|_K^2
+
2Cn\omega\left\{
R_\phi(f_\phi^*)
+
B_\phi\sqrt{\frac{\log(1/\delta)}{2n}}
\right\}
+
2Cm(1-\omega)\bar{\mathcal{E}}_m.
\end{equation}
Then $\Lambda_{n,\delta}^{\circ}$ is nonrandom and, with probability at least $1-\delta$,
\[
\|\hat f\|_K\le \Lambda_{n,\delta}^{\circ}.
\]
\end{corollary}

\begin{proof}
By Lemma~\ref{lem:normbound}, we have
\[
\|\hat f\|_K^2
\le
\|f_\phi^*\|_K^2
+
2Cn\omega\,\hat R_\phi(f_\phi^*)
+
2Cm(1-\omega)\mathcal E_m^*.
\]
Since $f_\phi^*$ is fixed, Hoeffding's inequality implies that, with probability at least $1-\delta$,
\[
\hat R_\phi(f_\phi^*)
\le
R_\phi(f_\phi^*)
+
B_\phi\sqrt{\frac{\log(1/\delta)}{2n}}.
\]
Since $\mathcal E_m^*\le \bar{\mathcal{E}}_m$ holds deterministically, substituting both bounds into the previous display yields
\[
\|\hat f\|_K^2
\le
\Lambda_{n,\delta}^{\circ\,2}.
\]
\end{proof}

\subsection{A corrected global bound for C-EDSVM}
\label{sec:genbound}

We first state a uniform $n^{-1/2}$ bound. It is useful as a baseline and for comparison with the standard global Rademacher theory.

\begin{theorem}[Global surrogate-risk bound for C-EDSVM]
\label{thm:genbound}
Since $|f(\xx)|\le M_0+\kappa\Lambda_{n,\delta/4}^{\circ}$ for all $f\in\mathcal F(\Lambda_{n,\delta/4}^{\circ},M_0)$, the hinge loss is bounded on this class by
\[
B_{n,\delta/4} := 1 + M_0 + \kappa\Lambda_{n,\delta/4}^{\circ}.
\]
Then, for every $\delta\in(0,1)$, with probability at least $1-\delta$,
\begin{equation}
\label{eq:genbound}
R_\phi(\hat f)-R_\phi(f_\phi^*)
\le
\frac{\|f_\phi^*\|_K^2}{2Cn\omega}
+
\frac{1-\omega}{\omega}\frac{m}{n}\mathcal E_m^*
+
\frac{4(\kappa\,\Lambda_{n,\delta/4}^{\circ}+M_0)}{\sqrt n}
+(
B_{n,\delta/4}
+
B_\phi)\sqrt{\frac{\log(4/\delta)}{2n}},
\end{equation}
where $\Lambda_{n,\delta/4}^{\circ}$ is defined by \eqref{eq:det-radius} with $\delta$ replaced by $\delta/4$.
\end{theorem}

\noindent\textbf{Statistical interpretation.} The bound \eqref{eq:genbound} decomposes into three statistically transparent components. The first term, $\|f_\phi^*\|_K^2/(2Cn\omega)$, is a variance-type penalty for the complexity of the population comparator, shrunk by the factor $\omega$ relative to the standard SVM rate. The second term, $\frac{1-\omega}{\omega}\frac{m}{n}\mathcal{E}_m^*$, is a bias term measuring the cost of benchmark mismatch on the elite set: it vanishes if the benchmark is perfect ($\mathcal{E}_m^*=0$), and grows with the fraction of elite points $m/n$ and the misalignment $(1-\omega)/\omega$. The complexity term (third and fourth terms) is controlled by the benchmark-dependent radius $\Lambda_{n,\delta/4}^\circ$, which is strictly smaller than the standard SVM radius whenever $m\mathcal{E}_m^*< n\hat{R}_\phi(f_\phi^*)$. Thus, a high-quality benchmark on a not-too-large elite set simultaneously reduces the function-class complexity \emph{and} introduces only mild bias, yielding a net improvement over standard C-SVM.

\begin{proof}
We decompose
\begin{align}
R_\phi(\hat f)-R_\phi(f_\phi^*)
&=
\bigl\{R_\phi(\hat f)-\hat R_\phi(\hat f)\bigr\}
+
\bigl\{\hat R_\phi(\hat f)-\hat R_\phi(f_\phi^*)\bigr\}
+
\bigl\{\hat R_\phi(f_\phi^*)-R_\phi(f_\phi^*)\bigr\} \notag\\
&=: (I)+(II)+(III).
\label{eq:3termdecomp}
\end{align}

\noindent\textbf{Bounding $(I)$.}
Apply Corollary~\ref{cor:det-radius} at level $\delta/4$: with probability at least $1-\delta/4$,
$\|\hat f\|_K\le \Lambda_{n,\delta/4}^{\circ}$. Call this event $\mathcal A_1$. Under the intercept envelope, this implies $\hat f\in \mathcal F(\Lambda_{n,\delta/4}^{\circ},M_0)$, and
\[
(I)
\le
\sup_{f\in \mathcal F(\Lambda_{n,\delta/4}^{\circ},M_0)}
\bigl\{R_\phi(f)-\hat R_\phi(f)\bigr\}.
\]
The hinge loss is $1$-Lipschitz, so the contraction principle together with \eqref{eq:radball-intercept} gives
\[
\hat{\mathfrak R}_n\bigl(\{\phi(y\,f(\cdot)):f\in \mathcal F(\Lambda_{n,\delta/4}^{\circ},M_0)\}\bigr)
\le
\frac{\kappa\,\Lambda_{n,\delta/4}^{\circ}+M_0}{\sqrt n}.
\]
A standard symmetrisation and McDiarmid concentration argument (cf.~\citealp{bartlett2002rademacher}) then yields, on an event $\mathcal A_2$ of probability at least $1-\delta/4$,
\[
(I)
\le
\frac{4(\kappa\,\Lambda_{n,\delta/4}^{\circ}+M_0)}{\sqrt n}
+
B_{n,\delta/4}\sqrt{\frac{\log(4/\delta)}{2n}}.
\]

\noindent\textbf{Bounding $(II)$.}
From \eqref{eq:empirical-edsvm}, $\hat R_n^{\mathrm{ED}}(f)\ge \omega \hat R_\phi(f)$ for all $f\in\mathcal H_K$.
Optimality of $\hat f$ gives
\[
\frac12\|\hat f\|_K^2 + Cn\,\hat R_n^{\mathrm{ED}}(\hat f)
\le
\frac12\|f_\phi^*\|_K^2 + Cn\,\hat R_n^{\mathrm{ED}}(f_\phi^*).
\]
Dropping $\frac12\|\hat f\|_K^2\ge0$ from the left and dividing by $Cn$,
\[
\hat R_n^{\mathrm{ED}}(\hat f)
\le
\hat R_n^{\mathrm{ED}}(f_\phi^*) + \frac{\|f_\phi^*\|_K^2}{2Cn}.
\]
Since $\hat R_\phi(\hat f)\le \hat R_n^{\mathrm{ED}}(\hat f)/\omega$,
\[
\hat R_\phi(\hat f)
\le
\frac{\hat R_n^{\mathrm{ED}}(f_\phi^*)}{\omega}
+
\frac{\|f_\phi^*\|_K^2}{2Cn\omega}.
\]
Expanding $\hat R_n^{\mathrm{ED}}(f_\phi^*)=\omega\,\hat R_\phi(f_\phi^*)+(1-\omega)\frac{m}{n}\mathcal E_m^*$,
\[
(II)=\hat R_\phi(\hat f)-\hat R_\phi(f_\phi^*)
\le
\frac{1-\omega}{\omega}\frac{m}{n}\mathcal E_m^*
+
\frac{\|f_\phi^*\|_K^2}{2Cn\omega}.
\]
This bound is deterministic (no probability event is needed for $(II)$).

\noindent\textbf{Bounding $(III)$.}
Since $f_\phi^*$ is fixed, the losses $\phi(Y_if_\phi^*(\xx_i))$ are i.i.d.\ and bounded by $B_\phi$.
Hoeffding's inequality gives, on an event $\mathcal A_3$ of probability at least $1-\delta/4$,
\[
(III)
\le
B_\phi\sqrt{\frac{\log(4/\delta)}{2n}}.
\]

\noindent\textbf{Union bound.}
Three probabilistic events are in play: the radius event $\mathcal A_1$ from Corollary~\ref{cor:det-radius} at level $\delta/4$ (failure probability $\delta/4$); the McDiarmid concentration event $\mathcal A_2$ for $(I)$ at level $\delta/4$ (failure probability $\delta/4$); and the Hoeffding event $\mathcal A_3$ for $(III)$ at level $\delta/4$ (failure probability $\delta/4$). Since $\delta/4+\delta/4+\delta/4=3\delta/4\le\delta$, the union bound gives that all three hold simultaneously with probability at least $1-\delta$, and \eqref{eq:genbound} follows.
\qedhere
\end{proof}
\begin{remark}
\label{rem:useful}
The bound \eqref{eq:genbound} is informative, but not uniformly sharper than the standard C-SVM bound. Its benefit is regime-dependent. In particular, it is advantageous when the benchmark is accurate on the elite set, so that $\mathcal E_m^*$ is small, and when $m/n$ is moderate or small, so that the benchmark-bias term remains controlled. In that regime the deterministic radius $\Lambda_{n,\delta/4}^{\circ}$ can be smaller than its C-SVM analogue, while the additional benchmark term remains mild. Conversely, a poor benchmark or a very large elite subset can erase this gain. Thus the theory should be interpreted as quantifying a bias--variance trade-off induced by the benchmark, rather than as establishing uniform dominance over standard C-SVM.
\end{remark}

\subsection{Comparison with standard C-SVM}
\label{sec:vs-svm}

The next comparison is stated at the level of the empirical radius bound from Lemma~\ref{lem:normbound}; it clarifies when the benchmark reduces the effective RKHS radius relative to standard C-SVM.

\begin{corollary}[Empirical radius comparison under a good benchmark]
\label{cor:vs-svm}
Let
\[
\Lambda_n^{\mathrm{SVM}}
=
\Bigl[\|f_\phi^*\|_K^2+2Cn\,\hat R_\phi(f_\phi^*)\Bigr]^{1/2}.
\]
Then
\[
\Lambda_n^2-\bigl(\Lambda_n^{\mathrm{SVM}}\bigr)^2
=
-2Cn(1-\omega)\hat R_\phi(f_\phi^*) + 2Cm(1-\omega)\mathcal E_m^*.
\]
In particular, if $\mathcal E_m^*=0$, then
\begin{equation}
\label{eq:norm-improvement}
\Lambda_n\le \Lambda_n^{\mathrm{SVM}}.
\end{equation}
More generally, EDSVM yields a smaller empirical radius whenever
\[
m\,\mathcal E_m^* < n\,\hat R_\phi(f_\phi^*).
\]
\end{corollary}

\begin{proof}
By \eqref{eq:normboundstar},
\[
\Lambda_n^2
=
\|f_\phi^*\|_K^2
+
2Cn\omega\,\hat R_\phi(f_\phi^*)
+
2Cm(1-\omega)\mathcal E_m^*.
\]
By \eqref{eq:normboundSVM},
\[
\bigl(\Lambda_n^{\mathrm{SVM}}\bigr)^2
=
\|f_\phi^*\|_K^2
+
2Cn\,\hat R_\phi(f_\phi^*).
\]
Subtracting gives
\[
\Lambda_n^2-\bigl(\Lambda_n^{\mathrm{SVM}}\bigr)^2
=
-2Cn(1-\omega)\hat R_\phi(f_\phi^*)
+
2Cm(1-\omega)\mathcal E_m^*.
\]
If $\mathcal E_m^*=0$, the right-hand side is nonpositive, which yields \eqref{eq:norm-improvement}. More generally, the right-hand side is negative whenever
\[
m\,\mathcal E_m^* < n\,\hat R_\phi(f_\phi^*),
\]
and in that case $\Lambda_n<\Lambda_n^{\mathrm{SVM}}$.
\end{proof}

The last inequality makes precise when the benchmark helps at the level of the empirical radius bound: the average elite-set mismatch must be small enough relative to the comparator's empirical hinge risk.

\begin{remark}[Benchmark quality and its origin]
The quantity $\mathcal{E}_m^*$ measures elite-set slack agreement between
the supplied benchmark and the population minimizer $f_\phi^*$; it says
nothing about the global quality of the benchmark. In particular,
$\mathcal{E}_m^*=0$ does not require the benchmark to be statistically
optimal: it holds whenever the benchmark slacks happen to coincide with
those of $f_\phi^*$ on $\mathcal{S}^*$, which can occur even for a
domain- or customer-guided benchmark that departs substantially from
$f_\phi^*$ outside the elite set. This decoupling is by design---EDSVM
penalizes only elite-set slack deviations, so a practitioner-specified
benchmark that is trusted on $\mathcal{S}^*$ but unreliable elsewhere
incurs no additional penalty from the theory. The relevant question is
therefore not whether the benchmark is globally optimal, but whether
\[
  m\,\mathcal{E}_m^* \;<\; n\,\hat{R}_\phi(f_\phi^*),
\]
the condition from Corollary~\ref{cor:vs-svm}. In practice, both sides
are estimable: $\mathcal{E}_m^*$ is approximated by
$\frac{1}{m}\sum_{i=1}^m(\hat\xi_i - \xi_i^*)^2$, where $\hat\xi_i$
are the hinge slacks of a cross-validated reference fit, and
$n\,\hat{R}_\phi(f_\phi^*)$ is approximated by the empirical hinge risk
of the same fit. The ratio of left to right side serves as a
pre-training diagnostic: a ratio well below~$1$ indicates the benchmark
is likely to reduce the effective complexity, suggesting that small
values of $\omega$ are appropriate; a ratio near or above~$1$ recommends
keeping $\omega$ close to~$1$, effectively falling back toward standard
C-SVM.
\end{remark}

\subsection{A localized Rademacher bound for C-EDSVM}
\label{sec:localized}

The global bound \eqref{eq:genbound} scales as $\Lambda_n/\sqrt n$. Under a standard low-noise or Bernstein condition one can sharpen this to a localized bound of order $\Lambda_0^2/n$, which is the appropriate fast-rate scale for localized convex-margin classes. To avoid using a sample-dependent radius inside the localized class, we use the same nonrandom envelope $\bar{\mathcal{E}}_m$ introduced in Corollary~\ref{cor:det-radius} and define
\begin{equation}
\label{eq:Lambda0}
\Lambda_{0,\delta}^2
:=
\|f_\phi^*\|_K^2
+
2Cn\omega\left\{
R_\phi(f_\phi^*)
+
B_\phi\sqrt{\frac{\log(2/\delta)}{2n}}
\right\}
+
2Cm(1-\omega)\bar{\mathcal{E}}_m.
\end{equation}
Because $\bar{\mathcal{E}}_m$ is nonrandom and every other term involves only the fixed population minimizer $f_\phi^*$, the quantity $\Lambda_{0,\delta}$ is genuinely nonrandom. Under the intercept envelope, the corresponding deterministic working class is $\mathcal F(\Lambda_{0,\delta},M_0)$.

Define the localized excess-loss class
\[
\mathcal G(r)
=
\Bigl\{
(\xx,y)\mapsto \phi\!\bigl(yf(\xx)\bigr)-\phi\!\bigl(yf_\phi^*(\xx)\bigr):
f\in \mathcal F(\Lambda_{0,\delta},M_0),\;
\PP\!\bigl[\phi(Yf(\XX))-\phi\!\bigl(Yf_\phi^*(\XX)\bigr)\bigr]\le r
\Bigr\}.
\]
Let
\[
\Psi_n(r)
=
\E\sup_{g\in\mathcal G(r)}\left|\frac1n\sum_{i=1}^n \sigma_i g(\XX_i,Y_i)\right|.
\]

\begin{theorem}[Localized oracle inequality for C-EDSVM]
\label{thm:localized}
Assume the hinge excess loss satisfies a Bernstein condition on the deterministic class $\mathcal F(\Lambda_{0,\delta},M_0)$: there exists $A_\phi>0$ such that
\begin{equation}
\label{eq:bernstein}
\PP g^2 \le A_\phi\, \PP g,\qquad g\in \mathcal G(r),\ r>0.
\end{equation}
Let $r_n^*$ be any fixed point satisfying
\begin{equation}
\label{eq:fixedpoint}
\Psi_n(r)\le \frac{r}{16A_\phi}\qquad \text{for all } r\ge r_n^*.
\end{equation}
Then there exists a universal constant $c>0$ such that, with probability at least $1-\delta$,
\begin{equation}
\label{eq:localizedbound}
R_\phi(\hat f)-R_\phi(f_\phi^*)
\le
c\left\{
r_n^*
+
\frac{\|f_\phi^*\|_K^2}{2Cn\omega}
+
\frac{1-\omega}{\omega}\frac{m}{n}\mathcal E_m^*
+
\frac{\log(4/\delta)}{n}
\right\}.
\end{equation}
\end{theorem}

\begin{proof}
From the optimization argument in the proof of Theorem~\ref{thm:genbound},
\begin{equation}
\label{eq:localized-erm}
\hat R_\phi(\hat f)-\hat R_\phi(f_\phi^*)
\le
\frac{\|f_\phi^*\|_K^2}{2Cn\omega}
+
\frac{1-\omega}{\omega}\frac{m}{n}\mathcal E_m^*.
\end{equation}
By Lemma~\ref{lem:normbound} and Hoeffding's inequality on $\hat R_\phi(f_\phi^*)$ (together with the deterministic bound $\mathcal E_m^*\le\bar{\mathcal{E}}_m$), we have with
probability at least $1-\delta/2$ that $\|\hat f\|_K\le \Lambda_{0,\delta}$.
Under the intercept envelope, $\hat f$ therefore belongs to the fixed class $\mathcal F(\Lambda_{0,\delta},M_0)$. Conditional on this radius event, apply a localized empirical-process inequality for bounded excess-loss classes under a Bernstein condition, such as Theorem~3.3 of \citet{bartlett2005local}, at confidence level $\delta/2$. This yields a second event of probability at least $1-\delta/2$ on which
\[
\PP\ell_{\hat f}-\PP\ell_{f_\phi^*}
\le
c\left\{
r_n^*
+
\bigl(\PP_n\ell_{\hat f}-\PP_n\ell_{f_\phi^*}\bigr)
+
\frac{\log(4/\delta)}{n}
\right\},
\]
where $\PP_n g=n^{-1}\sum_{i=1}^n g(\XX_i,Y_i)$ and $\ell_f(\xx,y)=\phi(yf(\xx))$,
and $c>0$ is a universal constant. Intersecting the two events and using a union bound gives probability at least $1-\delta$. Substituting \eqref{eq:localized-erm} proves
\eqref{eq:localizedbound}. \qedhere
\end{proof}

Theorem~\ref{thm:localized} is the sharper finite-sample result for the hinge-based method. It shows that benchmarking affects the fast-rate bound through two channels only: the deterministic effective radius $\Lambda_0$, which enters the fixed point $r_n^*$, and the additive benchmark-bias term $((1-\omega)/\omega)(m/n)\mathcal E_m^*$. In particular, if the benchmark is good enough to shrink $\Lambda_0$ substantially and if the Bernstein condition holds, then EDSVM inherits a smaller local complexity than standard C-SVM.

The estimate $\Psi_n(r)\lesssim (\kappa\Lambda_{0,\delta}+M_0)\sqrt{r/n}$ follows by bounding the
supremum of $\Psi_n$ over the localized class $\mathcal G(r)$ using the affine-class Rademacher bound
\eqref{eq:radball-intercept} applied on the fixed class $\mathcal F(\Lambda_{0,\delta},M_0)$, together with the
standard peeling argument that extracts the $\sqrt{r}$ factor from the variance condition
\eqref{eq:bernstein}; see \citet{bartlett2005local}, Lemma~3.2. Substituting into the
fixed-point condition $\Psi_n(r)\le r/(16A_\phi)$ then yields
\[
r_n^*\lesssim \frac{A_\phi (\kappa\Lambda_{0,\delta}+M_0)^2}{n}.
\]
Hence, under \eqref{eq:bernstein},
\begin{equation}
\label{eq:localized-simple}
R_\phi(\hat f)-R_\phi(f_\phi^*)
\lesssim
\frac{A_\phi(\kappa\Lambda_{0, \delta}+M_0)^2}{n}
+
\frac{\|f_\phi^*\|_K^2}{2Cn\omega}
+
\frac{1-\omega}{\omega}\frac{m}{n}\mathcal E_m^*
+
\frac{\log(4/\delta)}{n}.
\end{equation}
This improves the global $O\bigl((\kappa\Lambda_{n,\delta/4}^{\circ}+M_0)/\sqrt n\bigr)$ term to the fast-rate scale $O\bigl(A_\phi(\kappa\Lambda_{0,\delta}+M_0)^2/n\bigr)$, up to universal constants.\\

\noindent\textbf{Statistical interpretation.} The localized bound \eqref{eq:localizedbound} upgrades the global $O(1/\sqrt{n})$ rate to the fast scale $O(A_\phi(\kappa\Lambda_{0,\delta}+M_0)^2/n)$ under a low-noise Bernstein condition. The benchmark enters only through the effective radius $\Lambda_{0,\delta}$ and the additive bias $((1-\omega)/\omega)(m/n)\mathcal{E}_m^*$. When the benchmark is sufficiently accurate, $\Lambda_{0,\delta}$ is smaller than its standard SVM analogue, so C-EDSVM inherits a smaller local complexity---meaning faster rates of learning---than plain C-SVM in the low-noise regime.

\subsection{Excess classification risk for C-EDSVM}
\label{sec:excessrisk}

The surrogate bound translates into a $0$--$1$ bound through classification calibration. Since the hinge loss is classification-calibrated, there exists a nondecreasing function $\psi:[0,\infty)\to[0,\infty)$ with $\psi(0)=0$ such that
\begin{equation}
\label{eq:calibtransform}
R_{0\text{-}1}(f)-R^*\le \psi^{-1}\!\bigl(R_\phi(f)-R_\phi^*\bigr),
\end{equation}
where $R_\phi^*=\inf_f R_\phi(f)$.
Under a linear calibration regime (e.g., a Tsybakov margin condition on $\PP$), the excess
classification risk inherits the same rate as the surrogate excess risk.

\begin{theorem}[Excess $0$--$1$ risk for C-EDSVM]
\label{thm:excessrisk}
Assume the conditions of Theorem~\ref{thm:localized}.
Suppose additionally that the chosen comparator satisfies $R_\phi(f_\phi^*)=R_\phi^*$, and that there exists $c_\phi>0$ and $u_0>0$ such that
\[
\psi(t)\ge c_\phi\,t \qquad\text{for all } 0\le t\le u_0,
\]
where $\psi$ is the calibration function in \eqref{eq:calibtransform}. Assume moreover that
\begin{equation}
\label{eq:small-excess-hinge}
c\left\{
r_n^*
+
\frac{\|f_\phi^*\|_K^2}{2Cn\omega}
+
\frac{1-\omega}{\omega}\frac{m}{n}\mathcal E_m^*
+
\frac{\log(4/\delta)}{n}
\right\}\le u_0,
\end{equation}
where $c$ is the constant from Theorem~\ref{thm:localized} (for example, this holds for all sufficiently large $n$ whenever the right-hand side tends to zero).
Then, with probability at least $1-\delta$,
\begin{equation}
\label{eq:excessrisk}
R_{0\text{-}1}(\hat f)-R^*
\le
\frac{c}{c_\phi}\left\{
r_n^*
+
\frac{\|f_\phi^*\|_K^2}{2Cn\omega}
+
\frac{1-\omega}{\omega}\frac{m}{n}\mathcal E_m^*
+
\frac{\log(4/\delta)}{n}
\right\},
\end{equation}
where $c$ is the universal constant from Theorem~\ref{thm:localized}.
\end{theorem}

\begin{proof}
Since $R_\phi(f_\phi^*)=R_\phi^*$ by assumption,
\[
R_\phi(\hat f)-R_\phi^*
=
R_\phi(\hat f)-R_\phi(f_\phi^*).
\]
Theorem~\ref{thm:localized} therefore gives, with probability at least $1-\delta$,
\[
R_\phi(\hat f)-R_\phi^*
\le
c\left\{
r_n^*
+
\frac{\|f_\phi^*\|_K^2}{2Cn\omega}
+
\frac{1-\omega}{\omega}\frac{m}{n}\mathcal E_m^*
+
\frac{\log(4/\delta)}{n}
\right\}.
\]
By \eqref{eq:small-excess-hinge}, the surrogate excess bound lies in the interval $[0,u_0]$, so the calibration inequality \eqref{eq:calibtransform} and the lower bound $\psi(t)\ge c_\phi t$ imply
\[
R_{0\text{-}1}(\hat f)-R^*
\le
\psi^{-1}\!\bigl(R_\phi(\hat f)-R_\phi^*\bigr)
\le
\frac{1}{c_\phi}\bigl(R_\phi(\hat f)-R_\phi^*\bigr).
\]
Combining the two displays yields \eqref{eq:excessrisk}. \qedhere
\end{proof}

\begin{remark}[On the universality assumption]
\label{rem:universality}
A sufficient condition for $R_\phi(f_\phi^*)=R_\phi^*$ is that the hypothesis class be rich enough to approximate the population hinge-risk minimizer arbitrarily well; universal RKHSs such as the Gaussian RBF kernel on compact $\mathcal X$ are a standard example. For non-universal kernels, the same bound holds with $R_\phi^*$ replaced by the in-class optimum $R_\phi(f_\phi^*)$, at the cost of the corresponding approximation error.
\end{remark}

\subsection{Fast classification rates under Tsybakov noise}
\label{sec:tsybakov}

We now record a standard regime under which the localized bound yields a fast excess classification-risk rate. Let
\[
\eta(\xx):=\PP(Y=1\mid \XX=\xx).
\]
\begin{assumption}[Tsybakov noise condition]
\label{ass:tsybakov}
The distribution satisfies a Tsybakov noise (or margin) condition with exponent $q\ge 0$ if there exists $C_q>0$ such that
\begin{equation}
\label{eq:tsybakov}
\PP\!\left(0<\bigl|2\eta(\XX)-1\bigr|\le t\right)\le C_q t^q,
\qquad t>0.
\end{equation}
\end{assumption}
This condition controls the mass of points near the Bayes decision boundary. In standard localized empirical-process arguments for convex classification losses, Assumption~\ref{ass:tsybakov} is a standard sufficient condition for the Bernstein-type low-noise regime used in localized analyses for the excess loss on suitably bounded function classes; see, e.g., \citet{bartlett2006convexity}. 

The next corollary states the resulting fast classification rate for C-EDSVM on the localized RKHS ball. It should be read as a consequence of Theorems~\ref{thm:localized} and~\ref{thm:excessrisk} once the Bernstein condition is available under \eqref{eq:tsybakov}.

\begin{corollary}[Fast excess classification risk under Tsybakov noise]
\label{cor:tsybakov-fast}
Assume that the conditions of Theorem~\ref{thm:localized} hold on the deterministic ball
\[
\mathcal F(\Lambda_{0,\delta},M_0),
\]
that the chosen comparator satisfies $R_\phi(f_\phi^*)=R_\phi^*$,
and suppose in addition that the distribution satisfies the Tsybakov noise condition \eqref{eq:tsybakov}. If the localized complexity satisfies
\[
\Psi_n(r)\lesssim (\kappa \Lambda_{0,\delta}+M_0)\sqrt{\frac{r}{n}},
\]
so that
\[
r_n^*\lesssim \frac{A_\phi (\kappa\Lambda_{0,\delta}+M_0)^2}{n},
\]
then there exists a universal constant $C>0$ such that, with probability at least $1-\delta$,
\begin{equation}
\label{eq:tsybakov-fast-rate}
R_{0\text{-}1}(\hat f)-R^*
\le
C\left\{
\frac{A_\phi (\kappa\Lambda_{0,\delta}+M_0)^2}{n}
+
\frac{\|f_\phi^*\|_K^2}{2Cn\omega}
+
\frac{1-\omega}{\omega}\frac{m}{n}\mathcal E_m^*
+
\frac{\log(4/\delta)}{n}
\right\}.
\end{equation}
\end{corollary}

\begin{proof}
By Theorem~\ref{thm:localized},
\[
R_\phi(\hat f)-R_\phi(f_\phi^*)
\le
c\left\{
r_n^*
+
\frac{\|f_\phi^*\|_K^2}{2Cn\omega}
+
\frac{1-\omega}{\omega}\frac{m}{n}\mathcal E_m^*
+
\frac{\log(4/\delta)}{n}
\right\}
\]
with probability at least $1-\delta$. Under the displayed bound on $\Psi_n(r)$, the fixed-point condition implies
\[
r_n^*\lesssim \frac{A_\phi (\kappa\Lambda_{0,\delta}+M_0)^2}{n}.
\]
Moreover, for the hinge loss the calibration transform is linear, i.e.\ $\psi(t)=t$, so $R_{0\text{-}1}(\hat f)-R^* \le R_\phi(\hat f)-R_\phi^*$; see \citet{bartlett2006convexity}, Theorem~2.
Since $R_\phi(f_\phi^*)=R_\phi^*$ by assumption,
\[
R_\phi(\hat f)-R_\phi^*
=
R_\phi(\hat f)-R_\phi(f_\phi^*),
\]
and the right-hand side is bounded by the localized surrogate bound above.
Substituting completes the proof.
\end{proof}

Under the Tsybakov condition, the localized analysis yields a fast classification scale of order $O\!\bigl(A_\phi(\kappa\Lambda_{0,\delta}+M_0)^2/n\bigr)$, up to constants, rather than the global scale $O\!\bigl((\kappa\Lambda_{n,\delta/4}^{\circ}+M_0)/\sqrt n\bigr)$. The benchmark is beneficial when it reduces the deterministic local radius $\Lambda_{0,\delta}$ enough to offset the additive benchmark-bias term
\[
\frac{1-\omega}{\omega}\frac{m}{n}\mathcal E_m^*.
\]
Thus, in low-noise problems with a sufficiently accurate elite benchmark, C-EDSVM can improve both the local complexity and the resulting classification error bound relative to standard C-SVM.

\subsection{Generalization analysis of LS-EDSVM}
\label{sec:ls-theory}

We now treat the least-squares variant. The main structural difference is that the primal objective directly controls both the RKHS norm and the empirical squared-slack risk, which leads to a joint radius-risk bound.  The analysis mirrors that of Section~\ref{sec:normbound}--\ref{sec:vs-svm}
with $\phi$ replaced by the squared hinge loss. The quantities
$\mathcal{E}_m^{*,\mathrm{LS}}$, $\Gamma_n^{\mathrm{LS}}$,
$\Gamma_{n,\delta}^{\mathrm{LS}}$, $\Gamma_n^{\mathrm{LS,SVM}}$, and
$B_{\mathrm{LS}}$ are the exact LS analogues of $\mathcal{E}_m^*$,
$\Lambda_n^2$, $\Lambda_{n,\delta}^{\circ\,2}$, $\Lambda_n^{\mathrm{SVM}}$,
and $B_\phi$ respectively; their definitions below follow the same
template with $\hat R_\phi$ replaced by $\hat R_{\mathrm{LS}}$.
The only structurally new quantity is the local Lipschitz constant $L_{n,\delta}^{\mathrm{LS}}$, which is needed because the squared hinge loss is not globally $1$-Lipschitz and so the contraction argument requires a radius-dependent bound. As in the hinge-loss analysis, we retain the unpenalized intercept but work under the same deterministic intercept envelope $|\beta_0|\le M_0$.

Let
\[
R_{\mathrm{LS}}(f)=\E\bigl[(1-Yf(\XX))_+^2\bigr],\qquad
\hat R_{\mathrm{LS}}(f)=\frac1n\sum_{i=1}^n (1-y_if(\xx_i))_+^2,
\]
and define
\[
f_{\mathrm{LS}}^*\in \argmin_{f\in\mathcal F} R_{\mathrm{LS}}(f).
\]
The empirical LS-EDSVM risk is
\begin{equation}
\label{eq:empirical-lsedsvm}
\hat R_n^{\mathrm{LS,ED}}(f)
=
\omega\,\hat R_{\mathrm{LS}}(f)
+
\frac{1-\omega}{n}\sum_{i=1}^m \bigl(\xi_i(f)-\xi_i^*\bigr)^2.
\end{equation}
The LS benchmark-quality functional is
\begin{equation}
\label{eq:benchqual-ls}
\mathcal E_m^{*,\mathrm{LS}}
=
\frac1m\sum_{i=1}^m \bigl(\xi_i(f_{\mathrm{LS}}^*)-\xi_i^*\bigr)^2.
\end{equation}
Set
\begin{equation}
\label{eq:Gamma-ls}
\Gamma_n^{\mathrm{LS}}
=
\|f_{\mathrm{LS}}^*\|_K^2
+
2Cn\omega\,\hat R_{\mathrm{LS}}(f_{\mathrm{LS}}^*)
+
2Cm(1-\omega)\mathcal E_m^{*,\mathrm{LS}},
\end{equation}
and define the deterministic high-probability counterpart, in which $\bar{\mathcal{E}}_m^{\mathrm{LS}}$ is any nonrandom constant satisfying $\mathcal E_m^{*,\mathrm{LS}}\le \bar{\mathcal{E}}_m^{\mathrm{LS}}$ almost surely (e.g.\ $\bar{\mathcal{E}}_m^{\mathrm{LS}}=\frac{1}{m}\sum_{i=1}^m(B_{\mathrm{LS}}+\xi_i^*)^2$):
\begin{equation}
\label{eq:Gamma-ls-det}
\Gamma_{n,\delta}^{\mathrm{LS}}
=
\|f_{\mathrm{LS}}^*\|_K^2
+
2Cn\omega\left\{
R_{\mathrm{LS}}(f_{\mathrm{LS}}^*)
+
B_{\mathrm{LS}}\sqrt{\frac{\log(1/\delta)}{2n}}
\right\}
+
2Cm(1-\omega)\bar{\mathcal{E}}_m^{\mathrm{LS}}.
\end{equation}

\subsubsection{Joint norm-risk control}

\begin{lemma}[Joint norm-risk bound for LS-EDSVM]
\label{lem:ls-normbound}
Let $\hat f_{\mathrm{LS}}$ minimize \eqref{eq:lSEDPrimal}. Then
\begin{equation}
\label{eq:ls-joint-bound}
\|\hat f_{\mathrm{LS}}\|_K^2 + 2Cn\omega\,\hat R_{\mathrm{LS}}(\hat f_{\mathrm{LS}})
\le
\Gamma_n^{\mathrm{LS}}.
\end{equation}
In particular,
\[
\|\hat f_{\mathrm{LS}}\|_K\le \sqrt{\Gamma_n^{\mathrm{LS}}},\qquad
\hat R_{\mathrm{LS}}(\hat f_{\mathrm{LS}})\le \frac{\Gamma_n^{\mathrm{LS}}}{2Cn\omega}.
\]
\end{lemma}

\begin{proof}
Optimality of $\hat f_{\mathrm{LS}}$ yields
\[
\frac12\|\hat f_{\mathrm{LS}}\|_K^2 + Cn\,\hat R_n^{\mathrm{LS,ED}}(\hat f_{\mathrm{LS}})
\le
\frac12\|f_{\mathrm{LS}}^*\|_K^2 + Cn\,\hat R_n^{\mathrm{LS,ED}}(f_{\mathrm{LS}}^*).
\]
Since $\hat R_n^{\mathrm{LS,ED}}(\hat f_{\mathrm{LS}})\ge \omega \hat R_{\mathrm{LS}}(\hat f_{\mathrm{LS}})$, multiplying by $2$ gives
\[
\|\hat f_{\mathrm{LS}}\|_K^2 + 2Cn\omega\,\hat R_{\mathrm{LS}}(\hat f_{\mathrm{LS}})
\le
\|f_{\mathrm{LS}}^*\|_K^2 + 2Cn\,\hat R_n^{\mathrm{LS,ED}}(f_{\mathrm{LS}}^*),
\]
and substituting \eqref{eq:empirical-lsedsvm} at $f_{\mathrm{LS}}^*$ yields \eqref{eq:ls-joint-bound}.
\end{proof}

\begin{corollary}[Deterministic high-probability radius for LS-EDSVM]
\label{cor:ls-det-radius}
Assume the squared hinge loss is bounded by $B_{\mathrm{LS}}$ on the range of margins achieved by $f_{\mathrm{LS}}^*$. Then $\Gamma_{n,\delta}^{\mathrm{LS}}$ as defined in \eqref{eq:Gamma-ls-det} is nonrandom and, with probability at least $1-\delta$,
\[
\|\hat f_{\mathrm{LS}}\|_K^2 \le \Gamma_{n,\delta}^{\mathrm{LS}}.
\]
\end{corollary}

\begin{proof}
By Lemma~\ref{lem:ls-normbound},
\[
\|\hat f_{\mathrm{LS}}\|_K^2
\le
\Gamma_n^{\mathrm{LS}}
=
\|f_{\mathrm{LS}}^*\|_K^2
+
2Cn\omega\,\hat R_{\mathrm{LS}}(f_{\mathrm{LS}}^*)
+
2Cm(1-\omega)\mathcal E_m^{*,\mathrm{LS}}.
\]
Since $f_{\mathrm{LS}}^*$ is fixed, Hoeffding's inequality gives, with probability at least $1-\delta$,
\[
\hat R_{\mathrm{LS}}(f_{\mathrm{LS}}^*)
\le
R_{\mathrm{LS}}(f_{\mathrm{LS}}^*)
+
B_{\mathrm{LS}}\sqrt{\frac{\log(1/\delta)}{2n}}.
\]
Since $\mathcal E_m^{*,\mathrm{LS}}\le \bar{\mathcal{E}}_m^{\mathrm{LS}}$ holds deterministically, substituting both bounds gives $\|\hat f_{\mathrm{LS}}\|_K^2 \le \Gamma_{n,\delta}^{\mathrm{LS}}$.
\end{proof}

\subsubsection{A corrected global bound for LS-EDSVM}

To control the squared loss we use the same intercept envelope $|\beta_0|\le M_0$. For the affine class
\[
\mathcal F_{\mathrm{LS}}(B,M_0):=\{\beta_0+h: |\beta_0|\le M_0,\ \|h\|_K\le B\},
\]
we have
\[
|f(\xx)|\le M_0 + \kappa B,\qquad f\in \mathcal F_{\mathrm{LS}}(B,M_0),
\]
so on the deterministic class $\mathcal F_{\mathrm{LS}}(\sqrt{\Gamma_{n,\delta}^{\mathrm{LS}}},M_0)$ the squared hinge loss is Lipschitz with constant
\begin{equation}
\label{eq:Ln-ls}
L_{n,\delta}^{\mathrm{LS}} = 2\bigl(1+M_0+\kappa\sqrt{\Gamma_{n,\delta}^{\mathrm{LS}}}\bigr),
\end{equation}
a quantity defined for each $\delta\in(0,1)$; in particular $L_{n,\delta/4}^{\mathrm{LS}}=2(1+M_0+\kappa\sqrt{\Gamma_{n,\delta/4}^{\mathrm{LS}}})$.

\begin{theorem}[Global surrogate-risk bound for LS-EDSVM]
\label{thm:genbound-ls}
Assume the squared hinge loss is bounded by $B_n^{\mathrm{LS}}$ on the range of margins attained by functions in the deterministic class $\mathcal F_{\mathrm{LS}}(\sqrt{\Gamma_{n,\delta/4}^{\mathrm{LS}}},M_0)$ and by $f_{\mathrm{LS}}^*$. Then, for every $\delta\in(0,1)$, with probability at least $1-\delta$,
\begin{equation}
\label{eq:genbound-ls}
R_{\mathrm{LS}}(\hat f_{\mathrm{LS}})-R_{\mathrm{LS}}(f_{\mathrm{LS}}^*)
\le
\frac{\|f_{\mathrm{LS}}^*\|_K^2}{2Cn\omega}
+
\frac{1-\omega}{\omega}\frac{m}{n}\mathcal E_m^{*,\mathrm{LS}}
+
\frac{4L_{n,\delta/4}^{\mathrm{LS}}\bigl(\kappa\sqrt{\Gamma_{n,\delta/4}^{\mathrm{LS}}}+M_0\bigr)}{\sqrt n}
+
2B_n^{\mathrm{LS}}\sqrt{\frac{\log(8/\delta)}{2n}},
\end{equation}
where $\Gamma_{n,\delta/4}^{\mathrm{LS}}$ and $L_{n,\delta/4}^{\mathrm{LS}}=2(1+M_0+\kappa\sqrt{\Gamma_{n,\delta/4}^{\mathrm{LS}}})$ are the radius and local Lipschitz constant evaluated at confidence level $\delta/4$.
\end{theorem}

\begin{proof}
We decompose
\[
R_{\mathrm{LS}}(\hat f_{\mathrm{LS}})-R_{\mathrm{LS}}(f_{\mathrm{LS}}^*)
=
(I)+(II)+(III),
\]
where
\begin{align*}
(I)&=R_{\mathrm{LS}}(\hat f_{\mathrm{LS}})-\hat R_{\mathrm{LS}}(\hat f_{\mathrm{LS}}),\\
(II)&=\hat R_{\mathrm{LS}}(\hat f_{\mathrm{LS}})-\hat R_{\mathrm{LS}}(f_{\mathrm{LS}}^*),\\
(III)&=\hat R_{\mathrm{LS}}(f_{\mathrm{LS}}^*)-R_{\mathrm{LS}}(f_{\mathrm{LS}}^*).
\end{align*}
By Corollary~\ref{cor:ls-det-radius} applied at level $\delta/4$, with probability at least $1-\delta/4$,
\[
\|\hat f_{\mathrm{LS}}\|_K\le \sqrt{\Gamma_{n,\delta/4}^{\mathrm{LS}}},
\]
so $\hat f_{\mathrm{LS}}$ belongs to the deterministic class
\[
\mathcal F_{\mathrm{LS}}(\sqrt{\Gamma_{n,\delta/4}^{\mathrm{LS}}},M_0).
\]
On this class the squared hinge loss is $L_{n,\delta/4}^{\mathrm{LS}}$-Lipschitz, so by contraction and \eqref{eq:radball-intercept}, with probability at least $1-\delta/4$,
\[
(I)\le
\frac{4L_{n,\delta/4}^{\mathrm{LS}}\bigl(\kappa\sqrt{\Gamma_{n,\delta/4}^{\mathrm{LS}}}+M_0\bigr)}{\sqrt n}
+
B_n^{\mathrm{LS}}\sqrt{\frac{\log(8/\delta)}{2n}}.
\]
Next, from \eqref{eq:empirical-lsedsvm},
$
\hat R_n^{\mathrm{LS,ED}}(f)\ge \omega \hat R_{\mathrm{LS}}(f),\quad f\in\mathcal H_K.
$
Using optimality,
\[
\frac12\|\hat f_{\mathrm{LS}}\|_K^2 + Cn\,\hat R_n^{\mathrm{LS,ED}}(\hat f_{\mathrm{LS}})
\le
\frac12\|f_{\mathrm{LS}}^*\|_K^2 + Cn\,\hat R_n^{\mathrm{LS,ED}}(f_{\mathrm{LS}}^*),
\]
which implies
\[
\hat R_n^{\mathrm{LS,ED}}(\hat f_{\mathrm{LS}})
\le
\hat R_n^{\mathrm{LS,ED}}(f_{\mathrm{LS}}^*) + \frac{\|f_{\mathrm{LS}}^*\|_K^2}{2Cn}.
\]
Hence
\[
\hat R_{\mathrm{LS}}(\hat f_{\mathrm{LS}})
\le
\frac{\hat R_n^{\mathrm{LS,ED}}(f_{\mathrm{LS}}^*)}{\omega}
+
\frac{\|f_{\mathrm{LS}}^*\|_K^2}{2Cn\omega}.
\]
Expanding $\hat R_n^{\mathrm{LS,ED}}(f_{\mathrm{LS}}^*)$ gives
\[
\hat R_n^{\mathrm{LS,ED}}(f_{\mathrm{LS}}^*)
=
\omega\,\hat R_{\mathrm{LS}}(f_{\mathrm{LS}}^*)
+
(1-\omega)\frac{m}{n}\mathcal E_m^{*,\mathrm{LS}},
\]
so
\[
(II)\le
\frac{1-\omega}{\omega}\frac{m}{n}\mathcal E_m^{*,\mathrm{LS}}
+
\frac{\|f_{\mathrm{LS}}^*\|_K^2}{2Cn\omega}.
\]

Finally, since $f_{\mathrm{LS}}^*$ is fixed, Hoeffding's inequality implies, with probability at least $1-\delta/4$,
\[
(III)\le
B_n^{\mathrm{LS}}\sqrt{\frac{\log(8/\delta)}{2n}}.
\]

Combining the three bounds: the radius event and the McDiarmid event for $(I)$ each have failure probability $\delta/4$, and the Hoeffding event for $(III)$ has failure probability $\delta/4$. Since $\delta/4+\delta/4+\delta/4=3\delta/4\le\delta$, the union bound gives that all three hold simultaneously with probability at least $1-\delta$, and \eqref{eq:genbound-ls} follows.
\end{proof}

\subsubsection{Localized bound for LS-EDSVM}

The least-squares analysis is parallel once one works on the deterministic affine class delivered by Corollary~\ref{cor:ls-det-radius} together with the intercept envelope. In particular, the localized excess-loss class is restricted to
\[
\mathcal F_{\mathrm{LS}}(\sqrt{\Gamma_{n,\delta}^{\mathrm{LS}}},M_0),
\]
and the contraction step uses the local Lipschitz constant \eqref{eq:Ln-ls}.

\begin{corollary}[Localized LS-EDSVM rate]
\label{cor:localized-ls}
Assume that the squared hinge excess loss on the deterministic class
\[
\mathcal F_{\mathrm{LS}}(\sqrt{\Gamma_{n,\delta/4}^{\mathrm{LS}}},M_0)
\]
satisfies a Bernstein condition analogous to \eqref{eq:bernstein} with constant $A_{\mathrm{LS}}$. Then, with probability at least $1-\delta$,
\begin{equation}
\label{eq:localized-ls}
R_{\mathrm{LS}}(\hat f_{\mathrm{LS}})-R_{\mathrm{LS}}(f_{\mathrm{LS}}^*)
\lesssim
\frac{A_{\mathrm{LS}}(L_{n,\delta/4}^{\mathrm{LS}})^2\bigl(\kappa\sqrt{\Gamma_{n,\delta/4}^{\mathrm{LS}}}+M_0\bigr)^2}{n}
+
\frac{\|f_{\mathrm{LS}}^*\|_K^2}{2Cn\omega}
+
\frac{1-\omega}{\omega}\frac{m}{n}\mathcal E_m^{*,\mathrm{LS}}
+
\frac{\log(4/\delta)}{n}.
\end{equation}
\end{corollary}

\begin{proof}
The proof repeats the argument of Theorem~\ref{thm:localized} on the deterministic class
\[
\mathcal F_{\mathrm{LS}}(\sqrt{\Gamma_{n,\delta/4}^{\mathrm{LS}}},M_0),
\]
which is a fixed set because $\Gamma_{n,\delta/4}^{\mathrm{LS}}$ is nonrandom and $M_0$ is deterministic.
Corollary~\ref{cor:ls-det-radius} at level $\delta/2$ yields the required radius control. Conditional on that event, the localized empirical-process bound is applied at confidence level $\delta/2$, and a union bound gives overall probability at least $1-\delta$. The optimization step is exactly the one established in the proof of Theorem~\ref{thm:genbound-ls}. Combining the resulting fixed-point bound with the local Lipschitz constant $L_{n,\delta/4}^{\mathrm{LS}}$ and the affine-class complexity term $\kappa\sqrt{\Gamma_{n,\delta/4}^{\mathrm{LS}}}+M_0$ gives \eqref{eq:localized-ls}.
\end{proof}

\subsubsection{Comparison with LS-SVM}

\begin{corollary}[Empirical complexity comparison under a good LS benchmark]
\label{cor:ls-vs-svm}
Let
\[
\Gamma_n^{\mathrm{LS,SVM}}
=
\|f_{\mathrm{LS}}^*\|_K^2 + 2Cn\,\hat R_{\mathrm{LS}}(f_{\mathrm{LS}}^*).
\]
Then
\[
\Gamma_n^{\mathrm{LS}}-\Gamma_n^{\mathrm{LS,SVM}}
=
-2Cn(1-\omega)\hat R_{\mathrm{LS}}(f_{\mathrm{LS}}^*)
+
2Cm(1-\omega)\mathcal E_m^{*,\mathrm{LS}}.
\]
In particular, if $\mathcal E_m^{*,\mathrm{LS}}=0$, then
\[
\Gamma_n^{\mathrm{LS}}\le \Gamma_n^{\mathrm{LS,SVM}}.
\]
More generally, LS-EDSVM improves the empirical joint complexity-risk parameter whenever
\[
m\,\mathcal E_m^{*,\mathrm{LS}} < n\,\hat R_{\mathrm{LS}}(f_{\mathrm{LS}}^*).
\]
\end{corollary}

\begin{proof}
By \eqref{eq:Gamma-ls},
\[
\Gamma_n^{\mathrm{LS}}
=
\|f_{\mathrm{LS}}^*\|_K^2
+
2Cn\omega\,\hat R_{\mathrm{LS}}(f_{\mathrm{LS}}^*)
+
2Cm(1-\omega)\mathcal E_m^{*,\mathrm{LS}}.
\]
Subtracting
\[
\Gamma_n^{\mathrm{LS,SVM}}
=
\|f_{\mathrm{LS}}^*\|_K^2 + 2Cn\,\hat R_{\mathrm{LS}}(f_{\mathrm{LS}}^*)
\]
gives the identity. The remaining claims follow immediately.
\end{proof}

\subsubsection{Excess classification risk for LS-EDSVM}

\begin{theorem}[Excess $0$--$1$ risk for LS-EDSVM]
\label{thm:excessrisk-ls}
Assume the conditions of Corollary~\ref{cor:localized-ls}.
Suppose additionally that the chosen comparator satisfies
$R_{\mathrm{LS}}(f_{\mathrm{LS}}^*)=R_{\mathrm{LS}}^*$ (for non-universal kernels, $R_{\mathrm{LS}}^*$ is replaced throughout by $R_{\mathrm{LS}}(f_{\mathrm{LS}}^*)$ at the cost of an additional approximation-error term), and that the calibration function for the squared hinge loss satisfies
\[
\psi_{\mathrm{LS}}(t)\ge c_{\mathrm{LS}}t
\qquad\text{for all } 0\le t\le u_{0,\mathrm{LS}}.
\]
Assume moreover that
\begin{equation}
\label{eq:small-excess-ls}
\frac{A_{\mathrm{LS}}(L_{n,\delta/4}^{\mathrm{LS}})^2\bigl(\kappa\sqrt{\Gamma_{n,\delta/4}^{\mathrm{LS}}}+M_0\bigr)^2}{n}
+
\frac{\|f_{\mathrm{LS}}^*\|_K^2}{2Cn\omega}
+
\frac{1-\omega}{\omega}\frac{m}{n}\mathcal E_m^{*,\mathrm{LS}}
+
\frac{\log(4/\delta)}{n}
\le u_{0,\mathrm{LS}}
\end{equation}
(for example, for all sufficiently large $n$ whenever the left-hand side tends to zero). Then, with probability at least $1-\delta$,
\begin{equation}
\label{eq:excessrisk-ls}
R_{0\text{-}1}(\hat f_{\mathrm{LS}})-R^*
\lesssim
\frac{1}{c_{\mathrm{LS}}}\left\{
\frac{A_{\mathrm{LS}}(L_{n,\delta/4}^{\mathrm{LS}})^2\bigl(\kappa\sqrt{\Gamma_{n,\delta/4}^{\mathrm{LS}}}+M_0\bigr)^2}{n}
+
\frac{\|f_{\mathrm{LS}}^*\|_K^2}{2Cn\omega}
+
\frac{1-\omega}{\omega}\frac{m}{n}\mathcal E_m^{*,\mathrm{LS}}
+
\frac{\log(4/\delta)}{n}
\right\}.
\end{equation}
\end{theorem}

\begin{proof}
Since $R_{\mathrm{LS}}(f_{\mathrm{LS}}^*)=R_{\mathrm{LS}}^*$ by assumption,
Corollary~\ref{cor:localized-ls} gives, with probability at least $1-\delta$,
\begin{align*}
R_{\mathrm{LS}}(\hat f_{\mathrm{LS}})-R_{\mathrm{LS}}^*
&=
R_{\mathrm{LS}}(\hat f_{\mathrm{LS}})-R_{\mathrm{LS}}(f_{\mathrm{LS}}^*)\\
&\lesssim
\frac{A_{\mathrm{LS}}(L_{n,\delta/4}^{\mathrm{LS}})^2\bigl(\kappa\sqrt{\Gamma_{n,\delta/4}^{\mathrm{LS}}}+M_0\bigr)^2}{n}
+
\frac{\|f_{\mathrm{LS}}^*\|_K^2}{2Cn\omega}
+
\frac{1-\omega}{\omega}\frac{m}{n}\mathcal E_m^{*,\mathrm{LS}}
+
\frac{\log(4/\delta)}{n}.
\end{align*}
By \eqref{eq:small-excess-ls}, the surrogate excess bound lies in $[0,u_{0,\mathrm{LS}}]$. Applying the calibration inequality
\begin{align*}
R_{0\text{-}1}(\hat f_{\mathrm{LS}})-R^*
&\le
\psi_{\mathrm{LS}}^{-1}\!\bigl(R_{\mathrm{LS}}(\hat f_{\mathrm{LS}})-R_{\mathrm{LS}}^*\bigr)\\
&\le
\frac{1}{c_{\mathrm{LS}}}\bigl(R_{\mathrm{LS}}(\hat f_{\mathrm{LS}})-R_{\mathrm{LS}}^*\bigr)
\end{align*}
and substituting the surrogate bound yields \eqref{eq:excessrisk-ls}.\qedhere
\end{proof}

\subsection{Benchmark quality and asymptotic rate preservation}
\label{sec:minimax}

The main finite-sample bounds contain the term $m\mathcal{E}_m^*/n$, which so far has been treated as a fixed quantity. We now give it concrete probabilistic content by modeling the benchmark as an estimator, and then state a conservative rate-preservation result: whenever the benchmark-bias term is asymptotically negligible relative to the baseline fast rate available for the corresponding localized hinge problem, C-EDSVM inherits that same rate.

\subsubsection{Rate of benchmark-quality decay}

The key tool is the following deterministic bound, which follows from the $1$-Lipschitz property of the hinge function and the reproducing property of $\mathcal{H}_K$.

\begin{proposition}[Benchmark-quality decay]
\label{prop:Em-rate}
Let $\tilde f=\tilde\beta_0+\tilde h$ with $\tilde h\in\mathcal{H}_K$ be the benchmark, and write $f_\phi^*=\beta_{0,\phi}^*+h_\phi^*$. Let
\[
\epsilon_N:=\sup_{\xx\in\mathcal{X}}|\tilde f(\xx)-f_\phi^*(\xx)|.
\]
Then
\begin{equation}
\label{eq:Em-uniform}
  \mathcal{E}_m^* \;\le\; \epsilon_N^2,
\end{equation}
and in particular
\begin{equation}
\label{eq:Em-rkhs}
  \epsilon_N \;\le\; |\tilde\beta_0-\beta_{0,\phi}^*| + \kappa\|\tilde h - h_\phi^*\|_K.
\end{equation}
If $m=O(n^\alpha)$ for some $\alpha\in[0,1]$ and
\[
|\tilde\beta_0-\beta_{0,\phi}^*| + \|\tilde h-h_\phi^*\|_K = O_P(N^{-r})
\qquad (N\to\infty),
\]
then
\[
  \frac{m\mathcal{E}_m^*}{n} = O_P\!\left(n^{\alpha-1}N^{-2r}\right)
  \;\xrightarrow{\;P\;}\; 0
  \quad\text{whenever } N\gg n^{\alpha/(2r)}.
\]
\end{proposition}

\begin{proof}
By the $1$-Lipschitz property of $(1-z)_+$,
$|\xi_i(f_\phi^*)-\xi_i^*|\le|f_\phi^*(\xx_i)-\tilde f(\xx_i)|\le\epsilon_N$
for each $i\le m$, giving \eqref{eq:Em-uniform} after squaring and averaging.
For \eqref{eq:Em-rkhs},
\[
|\tilde f(\xx)-f_\phi^*(\xx)|
\le |\tilde\beta_0-\beta_{0,\phi}^*|+|(\tilde h-h_\phi^*)(\xx)|
\le |\tilde\beta_0-\beta_{0,\phi}^*|+\kappa\|\tilde h-h_\phi^*\|_K,
\]
by the reproducing property and Cauchy--Schwarz. The regime statement follows by substituting $\epsilon_N^2=O_P(N^{-2r})$ into $m\mathcal{E}_m^*/n\le mn^{-1}\epsilon_N^2$.
\end{proof}

\begin{remark}[Three useful regimes]
\label{rem:three-regimes}
Proposition~\ref{prop:Em-rate} covers three practically relevant settings.
(i) {Independent benchmark dataset of size $N$}: under standard RKHS estimation rates, the condition $N\gg n^{\alpha/(2r)}$ ensures $m\mathcal{E}_m^*/n\to 0$.
(ii) {Sparse elite set}: if $m=o(n)$ and $\mathcal{E}_m^*\le\epsilon^2$ is bounded, then $m\mathcal{E}_m^*/n\to 0$ without any requirement that the benchmark be globally optimal.
(iii) {Customer-guided benchmark with shrinking mismatch}: if $|\xi_i^*-\xi_i(f_\phi^*)|\le\varepsilon_m$ for all $i\le m$ and $m\varepsilon_m^2=o(n)$, then again $m\mathcal{E}_m^*/n\to 0$.
\end{remark}

\subsubsection{Rate preservation under Tsybakov noise}

We now make the generic benchmark-free rate $a_n$ explicit under the Tsybakov noise condition \eqref{eq:tsybakov} and a corresponding tuning of $C_n$. The purpose of this subsection is conservative: rather than deriving a new minimax theorem from scratch, we show that if the benchmark-free localized hinge analysis yields the usual fast rate, then the additional benchmark term preserves that rate whenever it is asymptotically negligible.

\begin{assumption}[Tuning of $C_n$]
\label{ass:Cn}
Let $q>0$ be the Tsybakov exponent from Assumption~\ref{ass:tsybakov}, let $B>0$ satisfy $\|h_\phi^*\|_K\le B$, and let $R_\phi(f_\phi^*)>0$. Set
\begin{equation}
\label{eq:Cn-tuning}
C_n \;\asymp\; \frac{B^2}{R_\phi(f_\phi^*)}\,n^{-1/(q+1)}.
\end{equation}
\end{assumption}

Under Assumption~\ref{ass:Cn} and \eqref{eq:Lambda0}, the deterministic radius satisfies
\begin{equation}
\label{eq:Lambda-scale}
\Lambda_{0,\delta}^2
\;\asymp\;
B^2 + 2C_n n\omega R_\phi(f_\phi^*)
\;\asymp\;
2\omega B^2\,n^{q/(q+1)},
\end{equation}
where the second $\asymp$ uses $C_n n\asymp \frac{B^2}{R_\phi(f_\phi^*)}n^{q/(q+1)}$ from Assumption~\ref{ass:Cn}, so that $2C_n n\omega R_\phi(f_\phi^*)=2\omega B^2 n^{q/(q+1)}$, which dominates the constant $B^2$ term for large $n$.  Hence $(\kappa\Lambda_{0,\delta}+M_0)^2 \asymp n^{q/(q+1)}$ up to multiplicative constants depending on $\kappa$, $\omega$, $B$, and $M_0$. In the benchmark-free localized hinge analysis, Tsybakov noise yields the fast classification scale $n^{-q/(q+1)}$ under this tuning; see, for example, \citet{bartlett2005local}, \citet{koltchinskii2002empirical}, and \citet{tsybakov2004aggregation}. We therefore take
\begin{equation}
\label{eq:an-tsybakov}
a_n := n^{-q/(q+1)}
\end{equation}
as the benchmark-free target rate in the theorem below.

\begin{theorem}[Fast-rate preservation under Tsybakov noise]
\label{thm:minimax}
Let Assumptions~\ref{ass:tsybakov} and~\ref{ass:Cn} hold with $q>0$, and let $\omega\in[\omega_{\min},1]$ for some fixed $\omega_{\min}\in(0,1]$. Assume that the benchmark-free localized hinge analysis for the same function class and tuning yields the fast excess classification rate
\begin{equation}
\label{eq:benchmarkfree-fast}
R_{0\text{-}1}(\hat f_{\mathrm{base}})-R^* = O_{\PP}(a_n)
\qquad\text{with}\qquad a_n=n^{-q/(q+1)}.
\end{equation}
Assume also that the comparator is Bayes-surrogate optimal, namely
\begin{equation}
\label{eq:comparator-opt}
R_\phi(f_\phi^*) = R_\phi^*.
\end{equation}
If
\begin{equation}
\label{eq:benchmark-negligible}
\frac{m\mathcal{E}_m^*}{n} \;=\; o(a_n),
\end{equation}
then, for all sufficiently large $n$, with probability at least $1-\delta$,
\begin{equation}
\label{eq:minimax-rate}
R_{0\text{-}1}(\hat f)-R^*
\;\le\;
C\,a_n
\;+
C\,\frac{1-\omega}{\omega}\frac{m\mathcal{E}_m^*}{n}
\;+
C\,\frac{\log(4/\delta)}{n},
\end{equation}
where $C>0$ depends on $\omega_{\min}$ and the constants in the benchmark-free localized hinge bound. In particular,
\[
R_{0\text{-}1}(\hat f)-R^* = O_{\PP}(a_n)=O_{\PP}\!\left(n^{-q/(q+1)}\right).
\]
Thus EDSVM preserves the benchmark-free fast classification rate whenever the additional benchmark term is asymptotically negligible.
\end{theorem}

\begin{proof}
From Corollary~\ref{cor:tsybakov-fast} and the linear calibration of the hinge loss, for all sufficiently large $n$ such that the surrogate excess bound lies in the linear-calibration regime,
\[
R_{0\text{-}1}(\hat f)-R^*
\le
c\left\{
r_n^*
+\frac{\|f_\phi^*\|_K^2}{2C_n n\omega}
+\frac{1-\omega}{\omega}\frac{m\mathcal{E}_m^*}{n}
+\frac{\log(4/\delta)}{n}
\right\}.
\]
Under Assumption~\ref{ass:Cn}, the comparator term satisfies
\[
\frac{\|f_\phi^*\|_K^2}{2C_n n\omega}
\le
\frac{B^2}{2C_n n\omega_{\min}}
\asymp a_n.
\]
For the fixed-point rate: by \eqref{eq:benchmark-negligible}, condition $m\mathcal E_m^*/n=o(a_n)$ holds, and by Corollary~\ref{cor:vs-svm} this implies $\Lambda_{0,\delta}\le\Lambda_n^{\mathrm{SVM}}$ for large $n$ (since $m\mathcal E_m^*<n\hat R_\phi(f_\phi^*)$ eventually). Consequently, the EDSVM working class $\mathcal F(\Lambda_{0,\delta},M_0)$ is a subset of the corresponding benchmark-free class $\mathcal F(\Lambda_n^{\mathrm{SVM}},M_0)$ for large $n$. The fixed-point $r_n^*$ is therefore bounded by the fixed-point of the benchmark-free analysis on $\mathcal F(\Lambda_n^{\mathrm{SVM}},M_0)$, which by assumption \eqref{eq:benchmarkfree-fast} is $O(a_n)$. Hence $r_n^*=O(a_n)$ for the EDSVM estimator. By \eqref{eq:benchmark-negligible}, the benchmark term is $o(a_n)$. Substituting these bounds into the preceding display proves \eqref{eq:minimax-rate}. The final $O_{\PP}(a_n)$ conclusion follows immediately.
\end{proof}

The chain of implications is therefore:
\begin{align}
\label{eq:chain}
\underbrace{|\tilde\beta_0-\beta_{0,\phi}^*|+\|\tilde h-h_\phi^*\|_K\to 0}_{\text{consistent benchmark}}
&\;\Longrightarrow\;
\underbrace{\frac{m\mathcal{E}_m^*}{n}\to 0}_{\text{Prop.~\ref{prop:Em-rate}}}
\nonumber\\
&\;\Longrightarrow\;
\underbrace{R_{0\text{-}1}(\hat f)-R^*=O_{\PP}\!\left(n^{-q/(q+1)}\right)}_{\text{Thm.~\ref{thm:minimax} with }a_n=n^{-q/(q+1)}}
\end{align}

{Theorem~\ref{thm:minimax} delivers a key statistical guarantee. Incorporating benchmark information via EDSVM does {not} degrade the minimax-optimal classification rate, provided the benchmark is sufficiently accurate in the sense that $m\mathcal{E}_m^*/n = o_P(a_n)$. This condition is satisfied, for example, when the benchmark is estimated from an independent sample of size $N \gg n^{\alpha/(2r)}$, or when the elite set is sparse ($m = o(n)$). Finite-sample improvements may still arise through a smaller effective radius $\Lambda_{0,\delta}$. Thus, EDSVM offers a genuine statistical advantage. It can improve upon standard SVM in finite samples when the benchmark is good, while never sacrificing the optimal asymptotic rate.}

\begin{table}[ht]
\centering
\footnotesize
\caption{Summary of generalization guarantees. Global bounds require no localization.
Localized bounds additionally require a Bernstein-type low-noise condition on the
corresponding excess-loss class. Here $\gamma=m/n$, while $\Lambda_{0,\delta}$ and
$\Gamma_{n,\delta}^{\mathrm{LS}}$ denote the deterministic benchmark-dependent complexity
parameters for C-EDSVM and LS-EDSVM, respectively.}
\label{tab:theory-summary}
\setlength{\tabcolsep}{4pt}
\tiny
\begin{tabular}{lcccc}
\toprule
 & \textbf{C-SVM} & \textbf{LS-SVM} & \textbf{C-EDSVM} & \textbf{LS-EDSVM} \\
\midrule
\textbf{Effective radius / parameter}
& $\Lambda_n^{\mathrm{SVM}}$
& $\Gamma_n^{\mathrm{LS,SVM}}$
& $\Lambda_{0,\delta}$
& $\Gamma_{n,\delta}^{\mathrm{LS}}$ \\[4pt]

\textbf{Global complexity}
& $O\!\left(\dfrac{\Lambda_n^{\mathrm{SVM}}}{\sqrt n}\right)$
& $O\!\left(\dfrac{L_n^{\mathrm{LS}}\sqrt{\Gamma_n^{\mathrm{LS,SVM}}}}{\sqrt n}\right)$
& $O\!\left(\dfrac{\kappa\Lambda_{n,\delta/4}^{\circ}+M_0}{\sqrt n}\right)$
& $O\!\left(\dfrac{L_{n,\delta}^{\mathrm{LS}}(\kappa\sqrt{\Gamma_{n,\delta}^{\mathrm{LS}}}+M_0)}{\sqrt n}\right)$ \\[8pt]

\textbf{Localized complexity}
& $O\!\left(\dfrac{(\Lambda_n^{\mathrm{SVM}})^2}{n}\right)$
& $O\!\left(\dfrac{(L_n^{\mathrm{LS}})^2\Gamma_n^{\mathrm{LS,SVM}}}{n}\right)$
& $O\!\left(\dfrac{A_\phi(\kappa\Lambda_{0,\delta}+M_0)^2}{n}\right)$
& $O\!\left(\dfrac{A_{\mathrm{LS}}(L_{n,\delta}^{\mathrm{LS}})^2(\kappa\sqrt{\Gamma_{n,\delta}^{\mathrm{LS}}}+M_0)^2}{n}\right)$ \\[8pt]

\textbf{Benchmark bias}
& --- & ---
& $\dfrac{1-\omega}{\omega}\,\gamma\,\mathcal{E}_m^*$
& $\dfrac{1-\omega}{\omega}\,\gamma\,\mathcal{E}_m^{*,\mathrm{LS}}$ \\[8pt]

\textbf{Comparator penalty}
& --- & ---
& $\dfrac{\|f_\phi^*\|_K^2}{2Cn\omega}$
& $\dfrac{\|f_{\mathrm{LS}}^*\|_K^2}{2Cn\omega}$ \\[8pt]

\textbf{Sufficient condition for smaller radius}
& --- & ---
& $m\,\mathcal{E}_m^* < n\,\hat{R}_\phi(f_\phi^*)$
& $m\,\mathcal{E}_m^{*,\mathrm{LS}} < n\,\hat{R}_{\mathrm{LS}}(f_{\mathrm{LS}}^*)$ \\[4pt]

\bottomrule
\end{tabular}
\end{table}

\section{Simulation study}
\label{sec:numerical}

We now investigate the finite-sample performance of the proposed C-EDSVM and LS-EDSVM models in comparison with standard C-SVM \citep{scholkopf2002kernels}, LINEX-SVM \citep{ma2019linexsvm,tang2021mvlinex}, and LS-SVM \citep{suykens1999lssvm,suykens2002lssvm}.\ The goals are twofold: (i) to assess whether the elite-driven formulations can track the behaviour of one or more benchmark classifiers while maintaining competitive predictive performance; and (ii) to study how the additional slack-guidance term behaves in a nonlinear classification problem where the Bayes decision boundary is curved and cannot be well approximated by a single linear separator.

To keep the geometry transparent and to allow visualization of the decision boundaries, we follow a controlled two-dimensional experimental setting based on mixtures of Gaussian distributions, in the spirit of \citet{hastie2009esl}.\ Unless otherwise stated, all experiments use a shared identity covariance matrix
\[
  \boldsymbol{I}_2 = \begin{pmatrix} 1 & 0 \\ 0 & 1 \end{pmatrix},
\]
a total of $200$ observations in a balanced design ($100$ per class), and a random split of $70\%$ for training and $30\%$ for testing.\ The same train-test split is used across competing methods in each replicate to ensure a fair comparison.\  All optimization problems arising in the fitting of SVM-type models are solved using the \texttt{CVXR} package in \textsf{R}
\citep{cvxr}.

We construct two overlapping Gaussian mixtures, one for each class, labelled  ``$+1$'' and ``$-1$''.\  For the $+1$ class we generate $10$ centres
$\boldsymbol{p}_i$ from
\[
  \boldsymbol{p}_i \sim \mathcal{N}\big((1,0)^\top,\boldsymbol{I}_2\big),
  \qquad i=1,\ldots,10,
\]
and for the $-1$ class we generate $10$ centres $\boldsymbol{q}_i$ from
\[
  \boldsymbol{q}_i \sim \mathcal{N}\big((0,1)^\top,\boldsymbol{I}_2\big),
  \qquad i=1,\ldots,10.
\]
We regard the sets $\{\boldsymbol{p}_i\}_{i=1}^{10}$ and $\{\boldsymbol{q}_i\}_{i=1}^{10}$ as fixed throughout the simulation study.\ Conditionally on these centres, and for each class, we then generate $100$ observations by drawing $10$ points from
\[
  \mathcal{N}\big(\boldsymbol{p}_i,\boldsymbol{I}_2/5\big)
  \quad \text{(respectively, } \mathcal{N}\big(\boldsymbol{q}_i,\boldsymbol{I}_2/5\big)\text{)}
\]
for each $i=1,\ldots,10$.\  Thus, each class is modelled as a mixture of ten Gaussian clusters, yielding a moderately complex nonlinear boundary with overlapping regions where the conditional class probabilities are close to one half; see \citet[Chapter~2]{hastie2009esl} for a related construction.

The Bayes decision boundary can be obtained by equating the posterior probabilities of the two classes.\  For a generic point
$\boldsymbol{z}\in\mathbb{R}^2$, and after cancelling common normalizing constants on both sides, the decision boundary is implicitly defined by
\[
  \PP(+1)\sum_{i=1}^{10}
  \exp\!\Big(-\tfrac{5}{2}\|\boldsymbol{p}_i-\boldsymbol{z}\|^2\Big)
  \;=\;
  \PP(-1)\sum_{j=1}^{10}
  \exp\!\Big(-\tfrac{5}{2}\|\boldsymbol{q}_j-\boldsymbol{z}\|^2\Big),
\]
where we adopt equal class priors $\PP(+1)=\PP(-1)=1/2$ in the balanced design.\  The associated Bayes classifier is the best possible classifier for this problem.\  To quantify this benchmark, we approximate the Bayes rule numerically on a fine grid over the feature space, classify the grid points according to the Bayes decision function, and integrate the resulting misclassification probabilities.\ This yields a Bayes accuracy of approximately $0.85$, corresponding to a Bayes error of about $0.15$, which we use as a natural performance ceiling for all methods.

Hyperparameters for all models are tuned by grid search combined with 5-fold cross-validation on the training data.\  For the penalty parameter $C$ we consider a grid of powers of two from $2^{-3}$ to $2^{5}$. For LINEX-SVM, the asymmetry parameter $a$ in \eqref{eq:LINEX} is varied over the set $\{-1,-2,\ldots, -8\}$, allowing us to explore different degrees of penalization for positive and negative margin violations.\ For the elite-driven models, the trade-off parameter $\omega$ is chosen from $\{0.1,0.2,\ldots,0.9\}$, ranging from heavy reliance on benchmark slacks to almost purely data-driven behaviour. Kernel parameters (e.g., the bandwidth of the Gaussian RBF kernel) are tuned over a logarithmically spaced grid.\ In each case, the selected hyperparameters minimize the cross-validated misclassification error.

Model performance is summarized using standard classification metrics, including overall accuracy, sensitivity (recall), specificity, precision, F1-score, and area under the ROC curve (AUC). We also report the area under the precision-recall curve (PR-AUC), which is particularly informative about performance on the positive class and complements AUC even in balanced designs.\ By examining these metrics jointly, and by comparing the estimated decision boundaries to the numerically approximated Bayes rule, we can assess not only the raw predictive performance of C-EDSVM and LS-EDSVM, but also the extent to which they preserve the elite behaviour induced by the benchmark models.

For the elite-driven models we explore several constructions of the benchmark slacks $\{\xi_i^*\}$. Let $\xi_{i,\mathrm{hinge}}$, $\xi_{i,\mathrm{LINEX}}$ and $\xi_{i,\mathrm{LS\text{-}SVM}}$ denote the slacks obtained from C-SVM, LINEX-SVM and LS-SVM, respectively. We consider four choices:
\begin{enumerate}[(i)]
\item a ``best--case'' benchmark:  $
    \xi_i^{*,\min} = \min\bigl\{
    \xi_{i,\mathrm{hinge}},\,
    \xi_{i,\mathrm{LINEX}},\,
    \xi_{i,\mathrm{LS\text{-}SVM}}
    \bigr\};$
\item an average benchmark based on LINEX-SVM and LS-SVM: 
$
    \xi_i^{*,\mathrm{mean}} =
    \tfrac{1}{2}\bigl(\xi_{i,\mathrm{LINEX}} + \xi_{i,\mathrm{LS\text{-}SVM}}\bigr);$ 
\item a conservative benchmark: 
 $
    \xi_i^{*,\max} = \max\bigl\{
    \xi_{i,\mathrm{hinge}},\,
    \xi_{i,\mathrm{LINEX}},\,
    \xi_{i,\mathrm{LS\text{-}SVM}}
    \bigr\};
 $
\item a LINEX-only benchmark: 
 $
    \xi_i^{*,\mathrm{LINEX}} = \xi_{i,\mathrm{LINEX}}.
 $
\end{enumerate}
In all cases the same elite set is used for C-EDSVM and LS-EDSVM, constructed from the union of the support vectors of the benchmark SVM-type models.\  The Bayes rule, approximated as above, is overlaid in all plots as an ideal reference.\ Figures~\ref{fig:balanced-min}-\ref{fig:balanced-linex} display the resulting decision boundaries.\  Each figure is organized as a $3\times 1$ panel.\  The left panel shows the training data together with the Bayes boundary (dashed grey curve), the C-SVM boundary (solid black), and the LINEX-SVM and LS-SVM boundaries (coloured curves). The middle panel overlays the Bayes boundary with the C-EDSVM decision contour, while the right panel shows the Bayes boundary together with LS-EDSVM. This layout makes it easy to see how the elite-driven models relate both to the Bayes rule and to the underlying benchmark SVMs.

In Figure~\ref{fig:balanced-min} the target slacks are given by the minimum of the three benchmark slacks.\ The C-EDSVM and LS-EDSVM boundaries  seem to approximately  follow the
Bayes rule  while remaining well within the envelope formed by the three reference SVMs.\  Intuitively, shrinking towards $\xi_i^{*,\min}$ pushes the elite-driven classifiers towards the ``best'' local margin configuration across the benchmarks.\  Numerically, both C-EDSVM and LS-EDSVM achieve an accuracy of $0.867$ and an F1-score of $0.846$, compared with $0.850$ and $0.836$ for C-SVM and LS-SVM; the AUC and PR-AUC are $0.921$ and $0.933$, respectively, again slightly improving on the classical SVMs. However, as we mentioned earlier, the primary objective of the EDSVM framework is not to guarantee accuracy gains in every scenario, but to provide a principled mechanism for integrating expert knowledge, benchmark models, and elite observations into the SVM learning process.\

\begin{figure}[ht]
\centering
\includegraphics[width=0.9\textwidth]{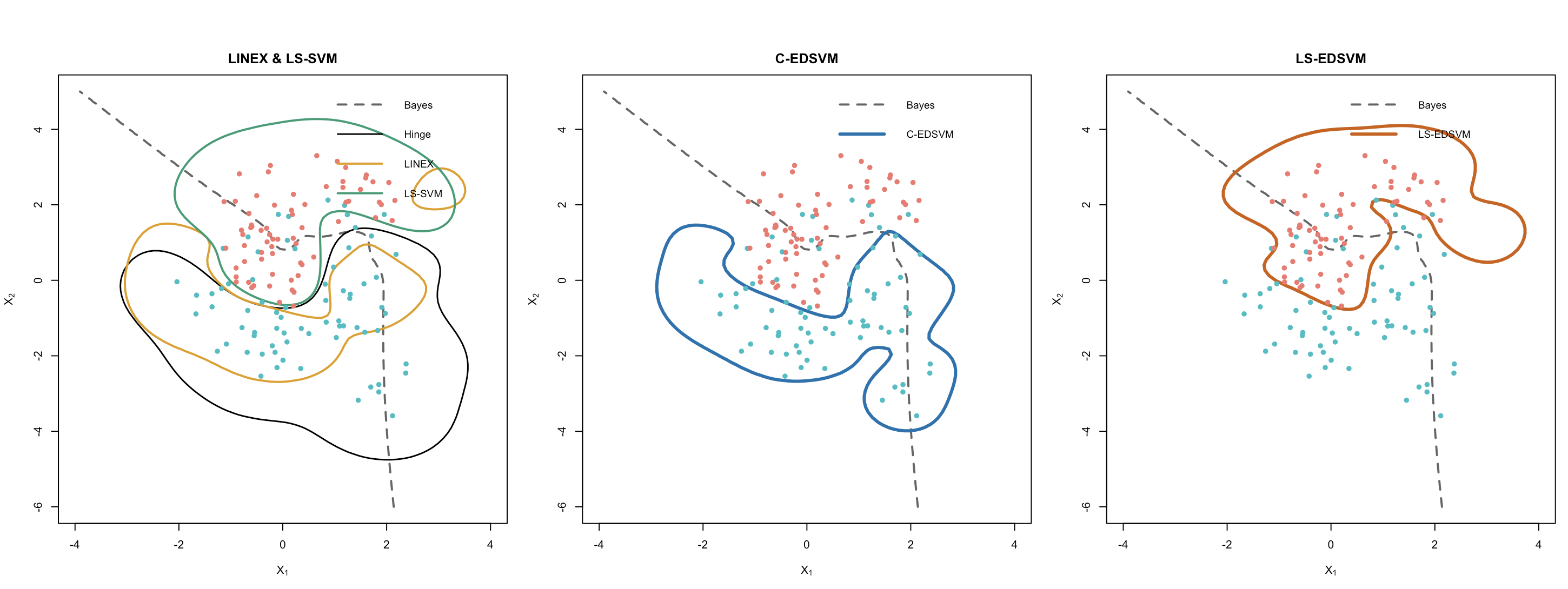}
\caption{Balanced nonlinear mixture: Bayes, hinge, LINEX-SVM and LS-SVM (left), and C-EDSVM (middle) and LS-EDSVM (right) when the elite target slack is $\xi_i^{*,\min} = \min\{\xi_{i,\mathrm{hinge}},
\xi_{i,\mathrm{LINEX}},\xi_{i,\mathrm{LS\text{-}SVM}}\}$.}
\label{fig:balanced-min}
\end{figure}

\begin{figure}[ht]
\centering
\includegraphics[width=0.9\textwidth]{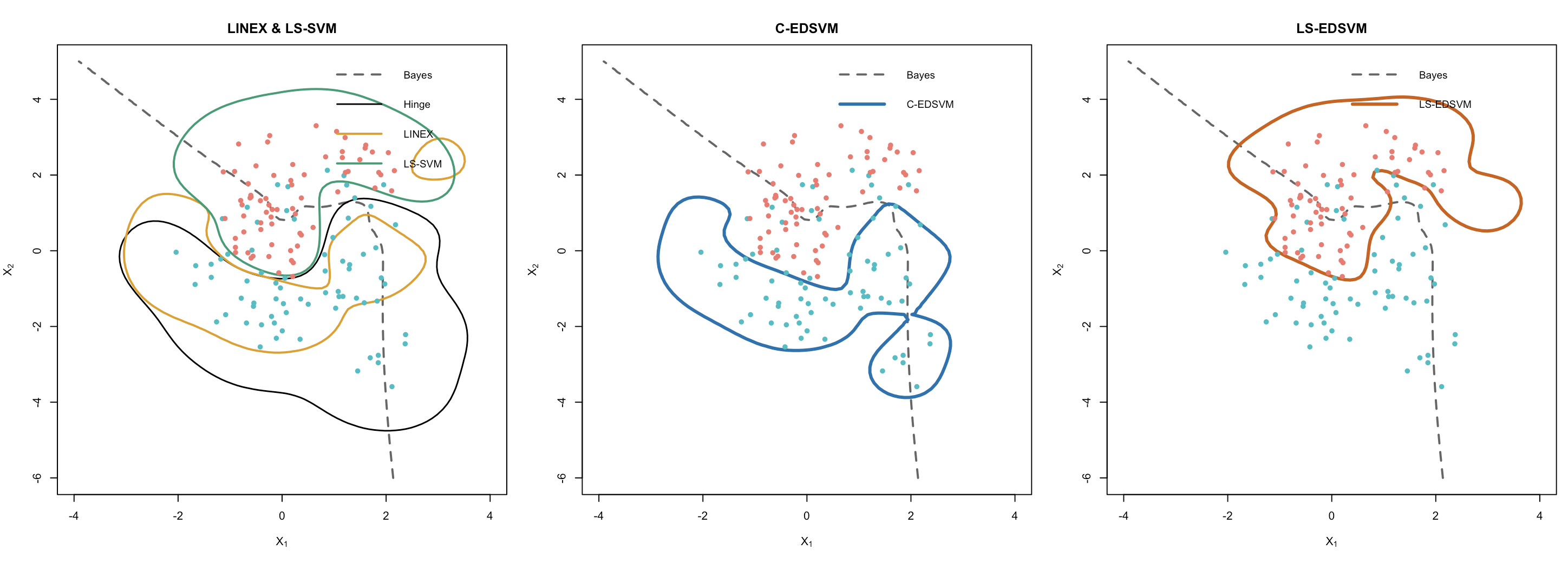}
\caption{Balanced nonlinear mixture: same layout as Figure~\ref{fig:balanced-min}, with elite target slack $\xi_i^{*,\mathrm{mean}} =
\frac{1}{2}(\xi_{i,\mathrm{LINEX}}+\xi_{i,\mathrm{LS\text{-}SVM}})$.}
\label{fig:balanced-mean2}
\end{figure}

Figure~\ref{fig:balanced-mean2} corresponds to the average benchmark $\xi_i^{*,\mathrm{mean}}$ based on LINEX-SVM and LS-SVM.\ In this case the elite-driven boundaries lie almost exactly between the two reference curves throughout most of the feature space, illustrating the ability of EDSVM to synthesize different modelling choices into a single compromise classifier.\  The performance of C-EDSVM and LS-EDSVM essentially matches that of the best benchmark SVMs: both methods attain an accuracy of $0.850$, an AUC of $0.919$, an F1-score of $0.824$, and a PR-AUC of $0.929$.

In Figure~\ref{fig:balanced-maxall} the target slack is the maximum of the three benchmarks.\  This more conservative choice encourages the elite-driven
boundaries to stay on the ``safe'' side of difficult observations, resulting in decision contours that are slightly more cautious near overlapping regions.\ Here C-EDSVM and LS-EDSVM again achieve an accuracy of $0.867$, but with higher discrimination metrics: an AUC of $0.932$ and a PR-AUC of $0.942$, compared with $0.918$ and $0.939$ for C-SVM and $0.939$ and $0.945$ for LS-SVM.\  The F1-score for the elite-driven models rises to $0.852$, indicating that the conservative slack target does not sacrifice overall balance between precision and recall.

\begin{figure}[ht]
\centering
\includegraphics[width=0.9\textwidth]{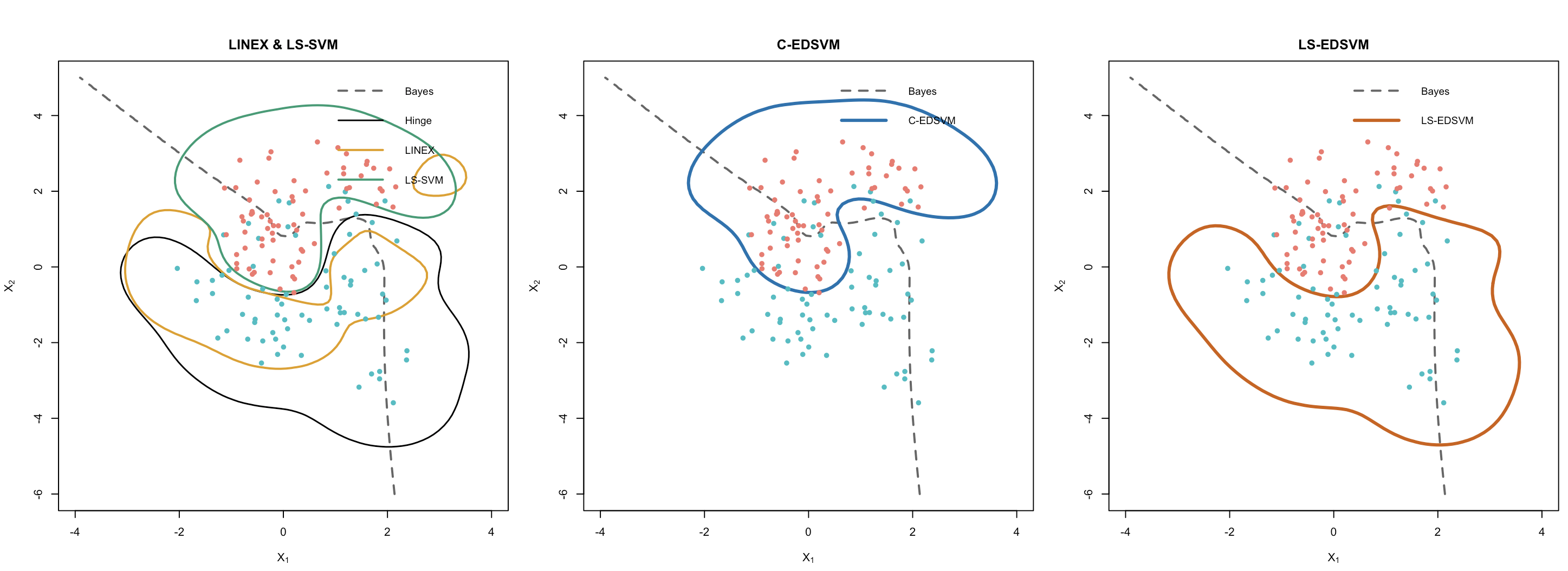}
\caption{Balanced nonlinear mixture: same layout as
Figure~\ref{fig:balanced-min}, with conservative target slack
$\xi_i^{*,\max} = \max\{\xi_{i,\mathrm{hinge}},
\xi_{i,\mathrm{LINEX}},\xi_{i,\mathrm{LS\text{-}SVM}}\}$.}
\label{fig:balanced-maxall}
\end{figure}

Finally, Figure~\ref{fig:balanced-linex} shows the case where the target slack comes solely from the LINEX-SVM benchmark, $\xi_i^{*,\mathrm{LINEX}}$.\  In this configuration the C-EDSVM and LS-EDSVM boundaries inherit much of the asymmetric shape induced by the LINEX loss, particularly in regions where the Bayes boundary is steep, but they still adapt to the data through the empirical slack term.\ The resulting numerical performance is very close to the standard SVMs: both EDSVM variants attain an accuracy of $0.850$, an AUC of $0.920$, an F1-score of $0.824$ and a PR-AUC of $0.925$.

\begin{figure}[ht]
\centering
\includegraphics[width=0.9\textwidth]{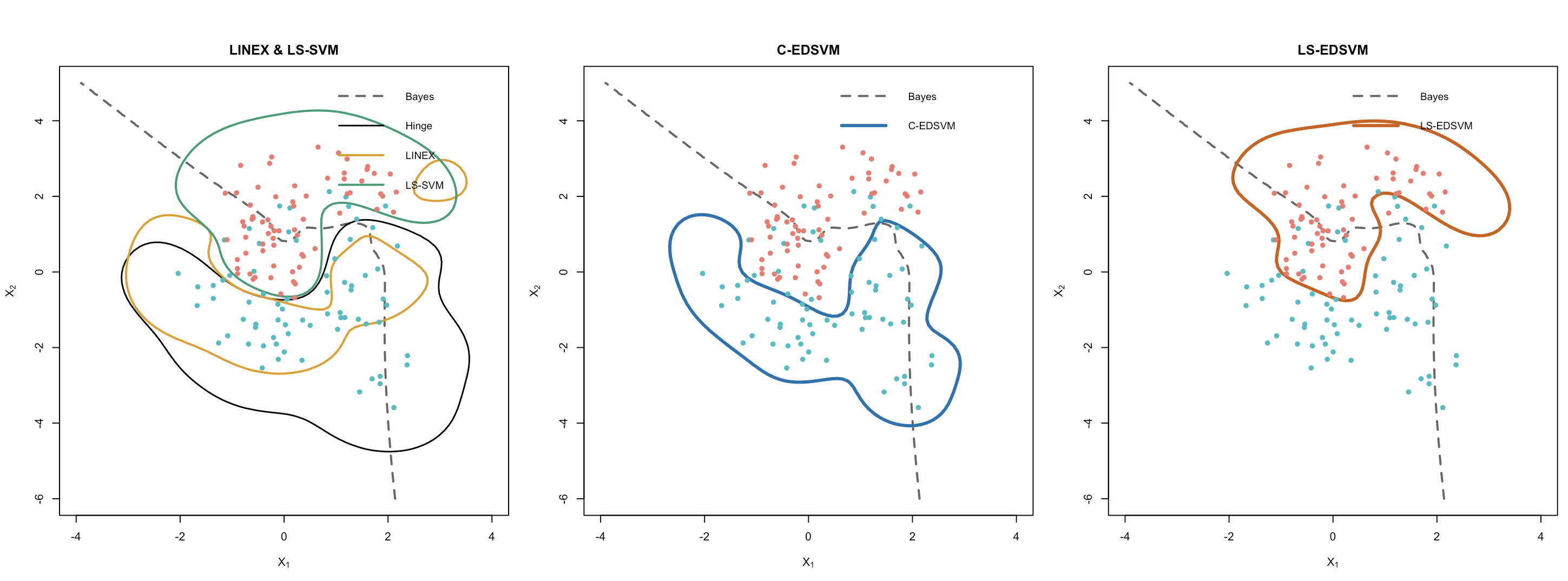}
\caption{Balanced nonlinear mixture: same layout as Figure~\ref{fig:balanced-min}, with LINEX-only target slack $\xi_i^{*,\mathrm{LINEX}} = \xi_{i,\mathrm{LINEX}}$.}
\label{fig:balanced-linex}
\end{figure}

Table~\ref{tab:balanced-slack} summarizes the numerical results across all four constructions of the target slack.\  Overall, C-EDSVM and LS-EDSVM either match or slightly improve upon the strongest benchmark SVMs in terms of accuracy, AUC, F1-score and PR-AUC. The differences between C-EDSVM and LS-EDSVM are negligible in this design, indicating that the main driver of elite-driven behaviour is the choice of target slack rather than the particular quadratic or linear form of the slack penalty.\  These findings suggest that slack-based elite guidance can be used flexibly, with different definitions of $\xi_i^*$, without compromising predictive performance.

\begin{table}[ht]
\centering
\caption{Balanced nonlinear mixture: performance of SVM-type methods under
different constructions of the elite target slack $\xi_i^*$.}
\label{tab:balanced-slack}
\begin{tabular}{llcccc}
\toprule
         Target $\xi_i^*$                  & Method     & Accuracy & AUC   & F1   & PR-AUC \\
\toprule
 & C-SVM     & 0.850 & 0.918 & 0.836 & 0.939 \\
 & LINEX-SVM & 0.767 & 0.897 & 0.696 & 0.914 \\
 & LS-SVM    & 0.850 & 0.939 & 0.836 & 0.945 \\
\midrule
$\xi^{*, \text{min}}$
 & C-EDSVM   & 0.867 & 0.921 & 0.846 & 0.933 \\
 & LS-EDSVM  & 0.867 & 0.921 & 0.846 & 0.933 \\
\midrule
$\xi^{*, \text{mean}}$
 & C-EDSVM   & 0.850 & 0.919 & 0.824 & 0.929 \\
 & LS-EDSVM  & 0.850 & 0.919 & 0.824 & 0.929 \\
\midrule
$\xi^{*, \text{max}}$
 & C-EDSVM   & 0.867 & 0.932 & 0.852 & 0.942 \\
 & LS-EDSVM  & 0.867 & 0.932 & 0.852 & 0.942 \\
\midrule
$\xi^{*, \text{LINEX}}$
 & C-EDSVM   & 0.850 & 0.920 & 0.824 & 0.925 \\
 & LS-EDSVM  & 0.850 & 0.920 & 0.824 & 0.925 \\
\bottomrule
\end{tabular}
\end{table}

%
%

\section{Real data analysis}
\label{sec:realdata}

To complement the simulation study, we assess the proposed EDSVM models on five benchmark datasets from the UCI repository.\  These datasets cover a range of sample sizes, levels of class imbalance, and feature types, and have been widely used to evaluate classification algorithms.\ The Pima dataset comprises $768$ observations on Pima Indian women aged $21$ years or older, with the task of predicting the presence of diabetes from eight clinical measurements, including glucose level, body mass index, and insulin \citep{smith1988pima}. The Australian dataset contains $690$ credit applications, of which $383$ are rejected and $307$ approved, described by $14$ mixed-type attributes and used for credit risk assessment \citep{quinlan1987credit,uci2019}.\  The Bupa-liver dataset consists of $345$ instances, $185$ of which correspond to patients with liver disorders; each observation includes six blood-test variables, such as ALP, ALT, and GGT, together with reported alcohol consumption \citep{forsyth1986bupa}.\ The Haberman dataset records $306$ cases of breast cancer surgery patients treated between 1958 and 1970, with predictors given by age, year of operation, and the number of positive axillary lymph nodes, and a binary outcome indicating whether the patient survived at least five years \citep{haberman1973dataset}.\  Finally, the Ionosphere dataset comprises $351$ radar returns collected at Goose Bay, Labrador; each record is described by $34$ continuous features derived from the radar signal, and the label ``good'' or ``bad'' indicates whether the return corresponds to a well-structured ionospheric target \citep{sigillito1989ionosphere}.

For each dataset we fit both linear and nonlinear SVM-type classifiers, using Gaussian RBF kernels as the nonlinear choice. Hyperparameters are selected by 5-fold cross-validation on the training data, using the same grids as in the simulation study.\  The construction of benchmark slacks is as follows: for C-EDSVM, the reference slack $\xi_i^*$ is defined as the average of the LINEX-SVM and LS-SVM slacks for observation $i$, whereas for LS-EDSVM we use the average of the C-SVM and LINEX-SVM slacks.\  This design ensures that the elite-driven models are guided by combinations of strong baseline methods, while still being estimated from the same data and tuning protocol as their benchmarks.

We report four performance metrics: accuracy, AUC, F1-score, and PR-AUC. The choice of which metric is emphasized for interpretation is purpose-driven and depends on which class is of primary interest in the application. For instance, in settings with a rare positive class F1-score and PR-AUC are more informative than overall accuracy.\  Consequently, some entries in the tables correspond to metrics that are not directly aligned with the primary objective for a given dataset and may therefore appear suboptimal; they are included to provide a complete picture of the behaviour of all methods across standard metrics.\ Tables~\ref{tab:uci-linear} and~\ref{tab:uci-rbf} summarize the results for linear and nonlinear classifiers, respectively.\  Each table reports 5-fold cross-validated means and standard deviations for all four metrics, with boldface indicating cases where proposed methods performed the best.

%
%

\subsection{Accuracy}

The rows labelled ``Acc.'' in Tables~\ref{tab:uci-linear} and~\ref{tab:uci-rbf} report the accuracy of the competing methods. In the linear setting, EDSVM models are frequently among the best performers. On the Australian dataset, C-EDSVM attains the highest accuracy, while on Bupa-liver LS-EDSVM yields the best result.\  On Haberman, LS-EDSVM again provides the top accuracy, slightly improving upon both C-SVM and LS-SVM. On Pima, LINEX-SVM achieves the highest accuracy, with both C-EDSVM and LS-SVM performing very similarly and marginally better than C-SVM. For Ionosphere, the classical C-SVM remains marginally superior, with C-EDSVM and LS-EDSVM close behind.

In the nonlinear case (Table~\ref{tab:uci-rbf}), the picture is similar.\ C-EDSVM achieves the highest accuracy on Australian and Haberman and remains competitive on the other datasets, while LS-EDSVM performs best on Bupa-liver.\ On Pima and Ionosphere, LS-SVM attains the highest accuracy, but here again the elite-driven models do not suffer large losses and often narrow the gap between different baseline losses.\   Overall, these results indicate that slack-based elite guidance can be layered on top of existing SVM formulations without sacrificing their good empirical performance in terms of accuracy.

\begin{table}[ht]
\centering
\caption{Classification performance (\%) of SVM models on UCI datasets with linear classifiers.\ Entries are 5-fold cross-validated means $\pm$ standard deviations.\ Boldface highlights the best-performing EDSVM variants; baseline methods are reported without emphasis, even when they attain the top score.}
\label{tab:uci-linear}
\renewcommand{\arraystretch}{1.2}
\setlength{\tabcolsep}{3pt}
\resizebox{\textwidth}{!}{%
\begin{tabular}{llccccc}
\toprule
Dataset & Metric & C-SVM & LINEX-SVM & LS-SVM & C-EDSVM & LS-EDSVM \\
\midrule
\multirow{4}{*}{Australian}
 & Acc.    & 84.78 $\pm$ 2.17 & 85.87 $\pm$ 2.54 & 85.14 $\pm$ 2.54 & \textbf{86.23 $\pm$ 2.17} & 85.51 $\pm$ 2.17 \\
 & AUC     & {93.71 $\pm$ 2.98} & 92.54 $\pm$ 2.79 & 92.46 $\pm$ 2.61 & 92.49 $\pm$ 2.72 & 92.35 $\pm$ 2.99 \\
 & F1      & 82.86 $\pm$ 1.77 & 83.81 $\pm$ 2.37 & 83.11 $\pm$ 2.29 & \textbf{84.13 $\pm$ 2.05} & 83.56 $\pm$ 1.85 \\
 & PR-AUC  & {83.24 $\pm$ 2.18} & 82.62 $\pm$ 3.22 & 82.59 $\pm$ 3.39 & 82.02 $\pm$ 4.85 & 81.50 $\pm$ 5.92 \\
\midrule
\multirow{4}{*}{Bupa-liver}
 & Acc.    & 68.84 $\pm$ 3.95 & 72.46 $\pm$ 3.45 & 67.54 $\pm$ 3.86 & 71.74 $\pm$ 2.49 & \textbf{73.19 $\pm$ 3.62} \\
 & AUC     & 67.84 $\pm$ 4.02 & 67.51 $\pm$ 4.52 & 68.17 $\pm$ 4.32 & 68.01 $\pm$ 4.73 & \textbf{69.84 $\pm$ 4.76} \\
 & F1      & 52.39 $\pm$ 3.61 & 52.50 $\pm$ 2.50 & 50.21 $\pm$ 6.31 & 54.33 $\pm$ 8.17 & \textbf{55.69 $\pm$ 9.53} \\
 & PR-AUC  & 31.27 $\pm$ 7.33 & 29.77 $\pm$ 9.92 & {33.64 $\pm$ 7.91} & 32.69 $\pm$ 9.89 & 33.08 $\pm$ 11.38 \\
\midrule
\multirow{4}{*}{Pima}
 & Acc.    & 71.65 $\pm$ 3.68 &{73.38 $\pm$ 4.55} & 73.16 $\pm$ 4.11 & 73.16 $\pm$ 4.76 & 72.94 $\pm$ 4.33 \\
 & AUC     & 80.51 $\pm$ 4.21 &{80.75 $\pm$ 3.84} & 80.27 $\pm$ 4.00 & 80.67 $\pm$ 3.88 & 80.59 $\pm$ 4.08 \\
 & F1      & {60.71 $\pm$ 7.31} & 57.86 $\pm$ 4.43 & 57.98 $\pm$ 4.03 & 57.45 $\pm$ 3.90 & 57.56 $\pm$ 3.79 \\
 & PR-AUC  & 59.57 $\pm$ 12.82 & {59.96 $\pm$ 10.79} & 59.26 $\pm$ 10.26 & 59.33 $\pm$ 11.54 & 58.58 $\pm$ 9.97 \\
\midrule
\multirow{4}{*}{Haberman}
 & Acc.    & 76.23 $\pm$ 0.82 & 79.51 $\pm$ 2.46 & 77.78 $\pm$ 0.82 & 78.68 $\pm$ 1.64 & \textbf{79.87 $\pm$ 2.46} \\
 & AUC     & 72.83 $\pm$ 5.67 & 75.46 $\pm$ 4.64 & 74.51 $\pm$ 5.28 & 74.70 $\pm$ 5.40 & \textbf{75.65 $\pm$ 4.07} \\
 & F1      & 44.47 $\pm$ 7.14 &{52.71 $\pm$ 6.55} & 26.67 $\pm$ 3.24 & 35.58 $\pm$ 10.58 & 26.14 $\pm$ 2.14 \\
 & PR-AUC  & 47.04 $\pm$ 8.30 & {52.01 $\pm$ 13.67} & 51.41 $\pm$ 13.45 & 51.27 $\pm$ 14.03 & 50.01 $\pm$ 10.31 \\
\midrule
\multirow{4}{*}{Ionosphere}
 & Acc.    & {92.14 $\pm$ 0.71} & 90.68 $\pm$ 2.20 & 90.71 $\pm$ 2.14 & 91.64 $\pm$ 1.71 & 92.06 $\pm$ 2.08 \\
 & AUC     & {92.68 $\pm$ 1.73} & 89.64 $\pm$ 3.59 & 92.37 $\pm$ 1.17 & 92.62 $\pm$ 1.57 & 91.90 $\pm$ 1.80 \\
 & F1      & 92.84 $\pm$ 2.59 & 93.60 $\pm$ 2.74 & 94.05 $\pm$ 3.19 & \textbf{94.86 $\pm$ 2.29} & 94.84 $\pm$ 2.30 \\
 & PR-AUC  & 76.00 $\pm$ 10.12 & 76.87 $\pm$ 8.24 & 77.73 $\pm$ 4.78 & \textbf{77.96 $\pm$ 5.82} & 76.98 $\pm$ 8.12 \\
\bottomrule
\end{tabular}
}
\end{table}

%
%

\begin{table}[ht]
\centering
\caption{Classification performance (\%) of SVM models on UCI datasets with nonlinear classifiers (Gaussian RBF kernel).\ Entries are 5-fold cross-validated means $\pm$ standard deviations.\ Boldface highlights the best-performing EDSVM variants; baseline methods are reported without emphasis, even when they attain the top score.}
\label{tab:uci-rbf}
\renewcommand{\arraystretch}{1.2}
\setlength{\tabcolsep}{3pt}
\resizebox{\textwidth}{!}{%
\begin{tabular}{llccccc}
\toprule
Dataset & Metric & C-SVM & LINEX-SVM & LS-SVM & C-EDSVM & LS-EDSVM \\
\midrule
\multirow{4}{*}{Australian}
 & Acc.    & 83.70 $\pm$ 1.09 & 70.29 $\pm$ 5.07 & 84.06 $\pm$ 0.72 & \textbf{85.14 $\pm$ 1.09} & 84.42 $\pm$ 1.81 \\
 & AUC     & 88.57 $\pm$ 2.20 & 84.96 $\pm$ 2.49 &{89.49 $\pm$ 2.87} & 89.01 $\pm$ 3.34 & 88.01 $\pm$ 2.07 \\
 & F1      & 79.87 $\pm$ 4.08 & 78.46 $\pm$ 4.60 & 80.25 $\pm$ 5.05 & \textbf{81.24 $\pm$ 4.43} & 80.94 $\pm$ 4.61 \\
 & PR-AUC  & 75.44 $\pm$ 8.77 & 74.45 $\pm$ 5.58 & 77.53 $\pm$ 6.29 & \textbf{80.83 $\pm$ 4.54} & 75.74 $\pm$ 4.03 \\
\midrule
\multirow{4}{*}{Bupa-liver}
 & Acc.    & 68.24 $\pm$ 3.55 & 67.39 $\pm$ 3.62 & 68.84 $\pm$ 3.92 & 68.12 $\pm$ 1.45 & \textbf{69.26 $\pm$ 2.90} \\
 & AUC     & {67.99 $\pm$ 2.88} & 67.57 $\pm$ 2.52 & 67.48 $\pm$ 2.45 & 67.26 $\pm$ 2.60 & 66.67 $\pm$ 1.19 \\
 & F1      & 60.13 $\pm$ 5.34 & 58.53 $\pm$ 3.49 & 58.61 $\pm$ 5.13 & \textbf{60.34 $\pm$ 4.83} & 58.76 $\pm$ 2.88 \\
 & PR-AUC  & 41.81 $\pm$ 17.99 & 36.75 $\pm$ 12.75 & 40.59 $\pm$ 15.76 & 41.86 $\pm$ 21.22 & \textbf{43.59 $\pm$ 19.55} \\
\midrule
\multirow{4}{*}{Pima}
 & Acc.    & 66.67 $\pm$ 2.81 & 70.78 $\pm$ 1.65 & {73.81 $\pm$ 4.11} & 71.33 $\pm$ 2.70 & 70.56 $\pm$ 1.43 \\
 & AUC     & 81.31 $\pm$ 3.35 & 79.74 $\pm$ 1.47 & 81.35 $\pm$ 3.42 & \textbf{81.84 $\pm$ 2.18} & 79.91 $\pm$ 2.46 \\
 & F1      & 23.02 $\pm$ 3.17 & 28.83 $\pm$ 13.67 & {63.19 $\pm$ 5.81} & 43.71 $\pm$ 10.83 & 42.03 $\pm$ 4.01 \\
 & PR-AUC  & {56.77 $\pm$ 4.42} & 45.72 $\pm$ 4.17 & 54.78 $\pm$ 4.87 & 55.65 $\pm$ 3.26 & 51.34 $\pm$ 2.92 \\
\midrule
\multirow{4}{*}{Haberman}
 & Acc.    & 77.05 $\pm$ 1.64 & 78.34 $\pm$ 1.64 & 76.23 $\pm$ 0.82 & \textbf{78.69 $\pm$ 1.64} & 77.05 $\pm$ 1.64 \\
 & AUC     &{72.64 $\pm$ 2.36} & 70.29 $\pm$ 2.13 & 68.66 $\pm$ 4.83 & 70.71 $\pm$ 3.23 & 68.50 $\pm$ 5.13 \\
 & F1      & 46.22 $\pm$ 1.78 & 37.65 $\pm$ 5.88 &{48.89 $\pm$ 4.37} & 37.78 $\pm$ 5.56 & 33.71 $\pm$ 4.76 \\
 & PR-AUC  & 43.66 $\pm$ 12.91 & 45.69 $\pm$ 4.80 & 53.34 $\pm$ 8.87 & \textbf{53.58 $\pm$ 10.80} & 47.39 $\pm$ 6.14 \\
\midrule
\multirow{4}{*}{Ionosphere}
 & Acc.    & 92.14 $\pm$ 0.71 & 94.29 $\pm$ 1.43 & {95.06 $\pm$ 2.14} & 93.88 $\pm$ 3.34 & 94.29 $\pm$ 3.17 \\
 & AUC     &{96.24 $\pm$ 2.81} & 96.02 $\pm$ 2.23 & 95.12 $\pm$ 2.03 & 95.91 $\pm$ 2.34 & 95.75 $\pm$ 2.30 \\
 & F1      & 94.96 $\pm$ 2.62 & 96.22 $\pm$ 2.81 &{96.71 $\pm$ 3.29} & 96.29 $\pm$ 2.14 & 96.23 $\pm$ 2.19 \\
 & PR-AUC  & 93.25 $\pm$ 1.62 & 93.37 $\pm$ 2.52 & {93.70 $\pm$ 2.74} & 93.68 $\pm$ 3.39 & 90.97 $\pm$ 5.36 \\
\bottomrule
\end{tabular}
}
\end{table}

\subsection{AUC, F1-score, and PR-AUC}

The remaining rows in Tables~\ref{tab:uci-linear} and~\ref{tab:uci-rbf} describe the behaviour of the methods with respect to AUC, F1-score, and PR-AUC. For linear classifiers, LS-EDSVM achieves the highest AUC on Bupa-liver and Haberman, while C-EDSVM and LINEX-SVM are highly competitive on Australian and Pima. In terms of F1-score, C-EDSVM leads on Australian and Ionosphere, and LS-EDSVM on Bupa-liver, whereas LINEX-SVM performs best on Haberman. For PR-AUC, C-SVM tends to dominate on Australian and Ionosphere, with LS-SVM and LINEX-SVM performing strongly on Bupa-liver, Pima, and Haberman; in all cases the elite-driven models remain close to the best baseline.

For nonlinear kernels, C-EDSVM often improves upon C-SVM in terms of AUC and F1-score, particularly on Australian and Bupa-liver, while LS-EDSVM and LS-SVM tend to dominate on Bupa-liver and Ionosphere.\ On Pima and Haberman, LS-SVM provides the strongest F1-score, reflecting its good performance on the minority class, but C-EDSVM usually attains comparable or better PR-AUC, indicating competitive precision-recall trade-offs in regions of practical interest.\ Across datasets and kernel choices, the elite-driven models rarely underperform their benchmarks by a large margin and frequently offer modest gains on at least one of the metrics that are most relevant for the positive class.

Taken together, the simulation and real-data results suggest that slack-based elite guidance is a robust regularization mechanism.\ It preserves, and in some cases enhances, the predictive performance of standard SVM formulations across a range of metrics, while embedding them in a richer framework that explicitly links new models to trusted benchmark behaviour at the level of elite observations.

%
%

\section{Discussion and future work}
\label{sec:discussion}

We have introduced Elite-Driven Support Vector Machines (EDSVM), a flexible framework that incorporates benchmark guidance and elite information directly into the SVM training process via slack variables.\ The central idea is to define a composite slack-level loss that balances empirical fit with proximity to benchmark slacks on a curated set of elite observations, thereby localizing the influence of prior information to the most influential points near the decision boundary.\ Within this framework we derived two concrete formulations, C-EDSVM and LS-EDSVM, and obtained their dual representations, showing that elite guidance can be implemented within the familiar quadratic programming infrastructure of standard SVMs.\ Simulation studies and real data analyses indicate that EDSVM models typically achieve predictive performance that is comparable to or better than standard SVM variants, while anchoring the decision boundary to meaningful benchmark behaviour.\ The main contribution of EDSVM, however, is not limited to raw accuracy gains.\ By working at the slack level on an elite set, the framework provides a principled mechanism to encode expert or client preferences, to introduce controlled shrinkage toward trusted benchmark models, and to enhance interpretability by making explicit the connection between new decision rules and existing reference classifiers.

Several directions for future research naturally emerge from this work.\ A first line of extension concerns fairness-aware learning.\ Elite-driven ideas can be combined with fairness criteria such as demographic parity or equalized odds by constructing benchmark slacks that satisfy fairness constraints or by defining elite sets on sensitive subgroups.\  A second direction is Bayesian EDSVM, in which one places priors on $(\beta_0,\bbeta)$ and on slack deviations $(\xi_i-\xi_i^*)$ and uses posterior inference to quantify uncertainty about the decision boundary; approximate methods such as variational inference or MCMC could be used to make such models computationally feasible.\  A third avenue concerns adaptive elite selection.\  In the present work, elites are derived from the support vectors of benchmark models, but alternative constructions based on leverage scores, or local margin geometry may yield more refined sets of elite observations, and semi-supervised or active learning variants could exploit partial elite labelling.\ Finally, EDSVM ideas can be transported to multi-task and high-dimensional settings, where elite guidance might be shared across related tasks and combined with sparsity-inducing penalties or dimension-reduction techniques.

Overall, EDSVM provides a useful building block for a broader program aimed at integrating expert knowledge, fairness considerations, and interpretability into modern kernel-based learning methods, without abandoning the computational and theoretical advantages of the SVM paradigm.

\section*{Acknowledgment:} 
Mohammad Jafari Jozani gratefully acknowledges partial support from NSERC Canada.

\bibliographystyle{imsart-nameyear}

\end{document}